\documentclass{article}

\PassOptionsToPackage{numbers,compress}{natbib}
\usepackage[preprint]{neurips_2026}

\usepackage[utf8]{inputenc}
\usepackage[T1]{fontenc}
\usepackage{hyperref}
\usepackage{url}
\usepackage{booktabs}
\usepackage{amsfonts}
\usepackage{amsmath}
\usepackage{amssymb}
\usepackage{amsthm}
\usepackage{nicefrac}
\usepackage{microtype}
\usepackage{xcolor}
\usepackage{graphicx}
\usepackage{enumitem}
\usepackage{array}
\usepackage{longtable}

\newtheorem{theorem}{Theorem}
\newtheorem{proposition}{Proposition}
\newtheorem{corollary}{Corollary}

\hypersetup{
    pdftitle={The BatchNorm Illusion: Diagnosing Normalization Artifacts in Machine Unlearning Evaluation},
    pdfauthor={Aaryaman Kalani, Murari Mandal, Dhruv Kumar, Mohan Kankanhalli, Yash Sinha},
    colorlinks=true,
    linkcolor=blue!60!black,
    citecolor=blue!60!black,
    urlcolor=blue!60!black
}

\title{The BatchNorm Illusion: Diagnosing Normalization Artifacts in Machine Unlearning Evaluation}

\author{
Aaryaman Kalani$^{1}$
\quad
Murari Mandal$^{2}$
\quad
Dhruv Kumar$^{1}$
\quad
Mohan Kankanhalli$^{3}$
\quad
Yash Sinha$^{1}$
\\[4pt]
$^{1}$BITS Pilani
\qquad
$^{2}$KIIT Bhubaneswar
\qquad
$^{3}$National University of Singapore
\\[3pt]
\texttt{\{f20220488, dhruv.kumar, yash.sinha\}@pilani.bits-pilani.ac.in}
\\
\texttt{murari.mandalfcs@kiit.ac.in}
\qquad
\texttt{dcsmsk@nus.edu.sg}
}

\begin{document}

\maketitle

\begin{abstract}
Approximate machine unlearning aims to remove the influence of specific training data from a trained model without retraining from scratch. We identify a previously undocumented confound in how unlearning is evaluated on BatchNorm-based architectures: a single forward pass over retain data, an operation that modifies no weight, can deterministically rewrite the model's normalization state and reverse the apparent surface-metric forgetting. We formalize this operation as a weight-preserving fixed-point operator and prove that any pre-versus-post gap it induces is provably attributable to BN running statistics rather than to any modification the unlearning method made to the weights. This attribution claim cleanly separates measurement failure (BN artifact) from encoder failure (residual weight-encoded information, recently documented in concurrent work), and the same operator framework yields a unique decomposition of linear-probe elevation into BN-measurement-bias and encoder-geometry components. Empirically, the artifact reverses headline forget accuracy by up to $78$pp across nine evaluated methods on standard benchmarks; an attacker with as few as $10$ unlabeled images recovers most of the masked accuracy; and a strict GroupNorm control reduces the artifact to zero across all methods. The tested membership-inference attacks change little under recalibration, locating the observed evaluation failure in forget accuracy and linear probing.
\end{abstract}

\section{Introduction}
\label{sec:intro}

Approximate machine unlearning seeks to remove the influence of specified training data from a model without retraining from scratch~\citep{cao2015towards,bourtoule2021machine}. A representative recent method~\citep{fan2024salun} reports forget-class accuracy of $3.0\%$ on CIFAR-10 after unlearning. We load the released checkpoint, change nothing in its weights, run \texttt{torch.optim.swa\_utils.update\_bn} on the retain set, and obtain $65.1\%$ on the same forget class. Which number measures forgetting? The trivial operation that produced the gap is the diagnostic this paper proposes. We are not aware of any systematic report of it as an evaluation control in the unlearning literature, and the headline forgetting numbers in six of nine methods we examine are recoverable when it is applied.

The unlearning literature is structured around two ways approximate methods can fail to satisfy GDPR-style erasure obligations. The first is \emph{encoder failure}: the model's parameters still encode the forgotten content even when classification accuracy on the forget set is near zero, as documented by relearning attacks~\citep{hu2024jogging,sheshadri2024latent} and concurrent work on linear probing~\citep{gao2026illusion,jeon2026erase}. The second is what we call \emph{measurement failure}: the evaluation procedure reports forgetting where none has occurred. This paper is about the second. The two failure modes are complementary, not competing; for several methods they coexist.

Two phenomena underlie the artifact. \emph{BN corruption} arises when BN is in train mode during unlearning: forget-data forward passes update running statistics through an exponential moving average that lies entirely outside the gradient computation graph, so weight-space gradient projection cannot intercept it. \emph{BN misalignment} arises when BN is frozen but the conv weights drift under projected ascent, so the activation distribution at each BN layer decouples from the still-frozen running statistics. Both surface as low forget accuracy; both vanish under recalibration.

Recalibration's attribution claim, that any pre/post gap is provably about BN running statistics rather than weights, rests on operator-theoretic properties: the operation is idempotent, weight-preserving, and pointwise equivalent to oracle population normalization at the model's existing weights. We extend the same framework to linear-probe measurements, deriving a unique decomposition of LP elevation into a BN-measurement-bias term and an encoder-geometry term. For several methods on CIFAR-100, the BN term is large and \emph{negative}: BN measurement bias \emph{deflates} the surface LP, masking encoder leakage that recalibration reveals. Concurrent work~\citep{gao2026illusion} attributing residual forget-class LP to Neural Collapse is consistent with our analysis but cannot quantify the encoder-level severity faithfully without first removing the BN measurement bias. The artifact is also adversarially exploitable: an attacker with local access to the unlearned checkpoint runs the same operation on as few as ten unlabeled images and recovers most of the masked forget-class accuracy.

\paragraph{Contributions.} We do not propose a new unlearning method; the contribution is in the genre of measurement corrections~\citep{magar2022data,shchur2018pitfalls,guo2017calibration}. \textbf{(C1)} Two mechanically opposite normalization artifacts, BN corruption and BN misalignment, mapping deterministically to BN train-vs-eval mode, with the corruption mechanism formalized as a non-interceptability result for any operator applied to gradients. \textbf{(C2)} BN recalibration as a near-zero-cost diagnostic, proven to be a deterministic weight-preserving fixed-point operator (Theorem~\ref{thm:t2}), together with a uniqueness decomposition (Proposition~\ref{prop:t3}) that separately identifies BN-measurement-bias and encoder-geometry components of linear-probe elevation, and shows that BN measurement bias \emph{masks} (rather than creates) encoder leakage for several methods. \textbf{(C3)} Empirical evaluation on nine evaluated methods on CIFAR-10/100 showing that the diagnostic restores near-perfect self-consistency between forget accuracy and relearning susceptibility, and that BN vulnerability is structurally entangled with retain preservation. \textbf{(C4)} An adversarial recovery experiment converting the artifact into a security finding, plus mechanistic falsifications via GroupNorm (architecture-controlled) and LayerNorm (ViT) controls.

\section{Related Work}
\label{sec:related}

\paragraph{Approximate unlearning.}
The unlearning problem was formalized in~\citet{cao2015towards}; subsequent work spans exact methods (SISA~\citep{bourtoule2021machine}) and approximate methods including data deletion~\citep{ginart2019making,izzo2021approximate}, analysis of unlearning factors~\citep{thudi2022unrolling}, fine-tuning on retain~\citep{golatkar2020eternal}, distillation from competent and incompetent teachers~\citep{chundawat2023can,tarun2023fast}, Fisher-weighted forgetting~\citep{golatkar2020eternal}, amnesiac learning~\citep{graves2021amnesiac}, saliency-based weight erasure~\citep{fan2024salun}, SCRUB~\citep{kurmanji2023towards}, selective synaptic dampening~\citep{foster2024fast}, projected gradient unlearning~\citep{hoang2024learn}, output-distribution reweighting (RWFT~\citep{ebrahimpourboroojeny2025necessityoutputdistributionreweighting}), and representation-level unlearning~\citep{pour2024}. Three recent benchmarks evaluate unlearning at scale on BN-based architectures (the NeurIPS unlearning competition~\citep{triantafillou2024unlearning}, Deep Unlearn~\citep{cadet2024deepunlearn}, and MUBox~\citep{huang2025mubox}), and to our reading none explicitly reports a BN recalibration control. Recent verification methods such as UMA~\citep{uma2024} also rely on BN backbones; our diagnostic is immediately applicable.

\paragraph{Encoder-level evaluation.}
\citet{gao2026illusion} use linear probes to demonstrate residual forget-class information across approximate methods, attributing the gap to feature--classifier misalignment via Neural Collapse. \citet{jeon2026erase} show that backbone freezing followed by classifier retraining recovers forget accuracy. Relearning attacks~\citep{hu2024jogging,sheshadri2024latent} establish that any unlearned representation in current methods can be cheaply re-induced by fine-tuning. These are encoder-failure findings: the model's parameters still encode the forgotten content. Our work is mechanistically distinct: it identifies a measurement failure that operates on top of, and largely orthogonal to, encoder failure. Our decomposition (Section~\ref{sec:diagnostic}) shows that for several methods on CIFAR-100, BN measurement bias \emph{masks} encoder leakage; the post-recalibration LP exceeds retrain by a larger margin than pre-recalibration LP could indicate. The Neural Collapse story is therefore strengthened, not refuted, by our diagnostic.

\paragraph{BatchNorm and recalibration elsewhere.}
BatchNorm's running statistics play a dual role (training-time stability~\citep{ioffe2015batch,santurkar2018how} and a checkpoint-time parametric object that determines inference behavior), which is what makes the artifact possible. Recomputing BN statistics with a target data pass is well known in other contexts: SWA~\citep{izmailov2018averaging} uses \texttt{update\_bn} to correct statistics of the averaged model; test-time adaptation~\citep{schneider2020improving,nado2020evaluating} similarly realigns BN with shifted test distributions. Our contribution is not the mechanism but (i) demonstrating that its absence exposes a systematic measurement failure in unlearning benchmarks, and (ii) the operator-theoretic characterization that makes the attribution unassailable.

\section{The Diagnostic and Its Properties}
\label{sec:diagnostic}

A model with Batch Normalization stores both learned weights and running activation statistics. Unlearning can leave these two parts of the checkpoint inconsistent: low forget-class accuracy may then reflect how the weights are normalized, rather than what they retain. Our diagnostic recomputes the running statistics from retain data at the current weights, using a single forward sweep with gradients disabled. Any resulting change in accuracy is therefore attributable to normalization state. The following results formalize this attribution and extend it to linear probing.

\subsection{Two mechanistically opposite phenomena}
\label{sec:phenomena}

A BN layer at position $\ell$ tracks running statistics $(\mu^\ell, (\sigma^\ell)^2)$ updated during forward passes in train mode via exponential moving average with momentum $m$:
\begin{equation}
\mu^\ell \leftarrow (1-m)\mu^\ell + m\hat\mu_{\mathrm{batch}}^\ell, \qquad
(\sigma^\ell)^2 \leftarrow (1-m)(\sigma^\ell)^2 + m\hat v_{\mathrm{batch}}^\ell.
\label{eq:ema}
\end{equation}
Here $\hat\mu_{\mathrm{batch}}^\ell$ and $\hat v_{\mathrm{batch}}^\ell$ are the batch moment estimates used for the running-buffer updates. At inference (eval mode), BN normalizes activations using $(\mu^\ell, (\sigma^\ell)^2)$ as fixed parameters.

\paragraph{Phenomenon 1 (BN corruption).}
If BN is in train mode during unlearning, every forget-data forward pass executes Eq.~\eqref{eq:ema} using batch statistics drawn from $\mathcal{D}_f$. Over $T$ unlearning steps, the running statistics drift toward the forget-class conditional distribution. Because Eq.~\eqref{eq:ema} occurs in the forward pass, no operator applied to gradients can intercept this update; this includes, in particular, projection-based unlearning methods that constrain weight gradients to be orthogonal to retain-class directions. The resulting model has weights consistent with retain-class behavior but BN running statistics that absorb forget-class information; predictions on retain-class inputs are normalized through forget-shifted statistics. We give the formal statement (Theorem~\ref{thm:t1}) and proof in Appendix~\ref{app:t1}.

\paragraph{Phenomenon 2 (BN misalignment).}
If BN is frozen (eval mode) during unlearning, the running statistics are immutable, but the conv weights $\theta$ shift via the unlearning update $\theta \to \theta_T$. The activation distribution $\mathcal{A}^\ell(\theta_T; \mathcal{D}_r)$ produced by the new weights on retain inputs is no longer the distribution against which $(\mu^\ell, (\sigma^\ell)^2)$ were calibrated. The BN layer normalizes incorrectly, producing a measurement-time distortion that depresses forget accuracy without genuine erasure: the weights still encode the forget-class structure, but they are evaluated under stale normalization. The two phenomena are mechanically opposite: the first contaminates the \emph{statistics} relative to fixed weights, the second contaminates the \emph{weights} relative to fixed statistics. Both produce indistinguishable surface metrics, and both vanish under BN recalibration, which we now define.

\paragraph{Why the two phenomena are not interconvertible.} BN corruption (Phenomenon 1) and BN misalignment (Phenomenon 2) are not endpoints of a continuum. They are produced by different control choices the unlearning method makes (BN train vs.\ eval), and they leave different traces in the saved checkpoint. Under Phenomenon 1, $\theta$ is approximately consistent with the retain distribution but $\varphi$ has drifted; the artifact lives in the running statistics. Under Phenomenon 2, $\theta$ has drifted but $\varphi$ has not; the artifact lives in the mismatch between the two. Both are surfaced by the same diagnostic because $\mathcal{R}^*$ corrects the running statistics to be consistent with the current $\theta$, regardless of which side of the pair was originally moved. A method can in principle exhibit both simultaneously (BN train mode for some steps, frozen for others), and the diagnostic remains a single forward pass.

\subsection{The recalibration operator}
\label{sec:operator}

\paragraph{The model space.}
A BN-equipped network is a pair $M = (\theta, \varphi)$, where $\theta$ collects all weight parameters (conv kernels, BN affine $\gamma, \beta$, classifier weights) and $\varphi = \{(\mu^\ell, (\sigma^\ell)^2)\}_{\ell=1}^L$ collects the running statistics. The forward map $f(x; \theta, \varphi)$ is a function of \emph{both}. Standard machine-learning rhetoric treats $\varphi$ as a passive bookkeeping artifact, but $\varphi$ is parametric: it is part of the saved checkpoint and it determines normalization at inference. An unlearning method may modify weights without modifying $\varphi$, or modify $\varphi$ without modifying weights, and forget accuracy and linear probing cannot distinguish the two cases without an explicit recalibration step. The membership-inference measurements in Appendix~\ref{app:mia} are largely unchanged by the same operation.

\paragraph{The recalibration operator.}
For any model $M = (\theta, \varphi)$ and retain distribution $\mathcal{D}_r$, define
\begin{equation}
\mathcal{R}^*_{\mathcal{D}_r}: (\theta, \varphi) \mapsto (\theta, \varphi^*(\theta; \mathcal{D}_r)),
\label{eq:recalop}
\end{equation}
where $\varphi^*(\theta; \mathcal{D}_r)$ are the population activation moments at every BN layer when the model with weights $\theta$ is evaluated on inputs drawn from $\mathcal{D}_r$ under oracle layer-wise normalization. Empirically, we use $\widehat{\mathcal{R}}_{N_r,b}$: a single forward sweep over $N_r$ retain images in train mode with gradients disabled and mini-batch size $b$ (PyTorch's \texttt{torch.optim.swa\_utils.update\_bn}). Its error separates sampling variation, $O_p(N_r^{-1/2})$, from a finite-batch term, $O_L(b^{-1})$, under the regularity conditions in Appendix~\ref{app:t2-proof}; the latter constant may depend on network depth $L$. The population operator below is exactly idempotent. The empirical pass is also idempotent when the input tensors, batch partition, and forward computation are held fixed.

\begin{theorem}[Recalibration as a fixed-point operator]
\label{thm:t2}
$\mathcal{R}^*_{\mathcal{D}_r}$ satisfies: \textbf{(i)} \emph{idempotence}: $\mathcal{R}^* \circ \mathcal{R}^* = \mathcal{R}^*$; \textbf{(ii)} \emph{weight invariance}: weights $\theta$ are unchanged; pre-BN activations at the first BN layer are identical to those of $M$, while activations at $\ell \geq 2$ are deterministic functions of unchanged $\theta$ and updated $\varphi^*$ (and may differ from $M$'s); \textbf{(iii)} \emph{pointwise oracle equivalence}: $f(x; \mathcal{R}^*(M))$ equals the network with $\theta$ and oracle population $\mathcal{D}_r$ moments at every BN layer, deterministically for every input; \textbf{(iv)} \emph{empirical concentration}: under the moment and second-order regularity assumptions of Appendix~\ref{app:t2-proof}, $\widehat{\mathcal{R}}_{N_r,b}$ deviates from $\mathcal{R}^*$ by $O_p(N_r^{-1/2}) + O_L(b^{-1})$ in any fixed norm on $\varphi$, for a fixed finite-depth network and independent retain images.
\end{theorem}

The proof and finite-batch derivation are in Appendix~\ref{app:t2-proof}. The pass changes normalization statistics but no weight, which gives the following attribution:

\begin{corollary}[BN attribution]
\label{cor:t2}
For any metric $g(M)$ depending on the forward map, the gap $g(\mathcal{R}^*(M)) - g(M)$ is provably attributable to BN state and not to weights. Equivalently, $\mathcal{R}^*$ cannot create forgetting and cannot destroy weight-encoded forgetting; it can only \emph{unmask} which regime the model was in.
\end{corollary}

Any metric change induced by the pass comes from normalization state, with the learned parameters fixed. ``Weight invariance'' refers to encoder weights $\theta$ (including BN affine $\gamma,\beta$), not to forward-pass features themselves; features are deterministic functions of $\theta$ and $\varphi$, so an update to $\varphi$ predictably modifies them.

\subsection{Decomposition of linear-probe elevation}
\label{sec:decomposition}

The previous subsection frames the artifact for the surface metric (forget accuracy). We now use the operator framework to disambiguate \emph{linear-probe} measurements, which concurrent work~\citep{gao2026illusion} interprets via Neural Collapse.

For an unlearned model $M$, let $\mathrm{LP}^M$ denote forget-class linear-probe accuracy, $M_{\mathrm{cal}} = \mathcal{R}^*(M)$ the recalibrated model, and $M_{\mathrm{retr}}$ the retrain-from-scratch reference. The next result splits the gap relative to retrain into a normalization-state contribution and a contribution that remains after recalibration.

\begin{proposition}[Unique BN/NC attribution decomposition]
\label{prop:t3}
For any $M$, the decomposition of the LP gap relative to retrain
\begin{equation}
\mathrm{LP}^M - \mathrm{LP}^{M_{\mathrm{retr}}} \;=\; \underbrace{\big(\mathrm{LP}^{M_{\mathrm{cal}}} - \mathrm{LP}^{M_{\mathrm{retr}}}\big)}_{\textstyle \mathrm{NC\text{-}residual}} \;+\; \underbrace{\big(\mathrm{LP}^M - \mathrm{LP}^{M_{\mathrm{cal}}}\big)}_{\textstyle \mathrm{BN\text{-}residual}}
\label{eq:decomp}
\end{equation}
is the unique decomposition into terms satisfying (i) the BN-residual vanishes whenever $\varphi = \varphi^*(\theta; \mathcal{D}_r)$, and (ii) the NC-residual is invariant under $\mathcal{R}^*$.
\end{proposition}

The proof (Appendix~\ref{app:t3-proof}) is short: existence is the algebraic identity $A - C = (A - B) + (B - C)$; uniqueness follows by evaluating any candidate split on $M_{\mathrm{cal}}$ and using the two criteria. Uniqueness matters because it precludes alternative attributions consistent with the same two criteria: any decomposition that respects ``BN-residual vanishes at $\varphi^*$'' and ``NC-residual is recalibration-invariant'' must coincide with Eq.~\eqref{eq:decomp}.

\begin{corollary}[BN measurement bias is identifiable independently of encoder geometry]
\label{cor:nc}
A nonzero BN-residual is a measurement bias attributable to the running statistics alone, not to feature--classifier misalignment in the sense of Neural Collapse. By Theorem~\ref{thm:t2} (weight invariance), $\mathcal{R}^*$ leaves encoder and classifier weights untouched, so the encoder-level forget-class structure is identical before and after recalibration; any LP change reflects a correction to the surface measurement of that fixed structure. A positive BN-residual indicates the surface LP \emph{overstates} encoder leakage; a negative BN-residual indicates the surface LP \emph{understates} it. The post-recalibration LP, $\mathrm{LP}^{M_{\mathrm{cal}}}$, is the unbiased measure of the encoder's residual forget-class linear separability.
\end{corollary}

The operational consequence is direction-dependent. The BN-residual on CIFAR-100 for GA, SalUn, and SSD is large and \emph{negative} (Table~\ref{tab:decomp-cifar}). BN measurement bias \emph{deflates} the pre-recalibration LP, hiding encoder-level leakage that recalibration reveals. The NC-residuals for the same methods are positive and substantial: the recalibrated LP exceeds retrain by $14$--$17$pp, indicating the encoder retains \emph{more} forget-class structure than retrain. The Neural Collapse story of~\citet{gao2026illusion} is consistent with our analysis on these methods, but their pre-recalibration measurement could not see that the BN measurement bias was masking the encoder-level severity by a magnitude that exceeds the (uncorrupted) NC signal itself.

We caution that an additive squared-energy (Pythagorean) form does not hold: an empirical cross-correlation analysis (Appendix~\ref{app:cross-corr}) finds mean $|\rho| = 0.40$ across train-mode methods, with maximum $0.68$, ruling out $\|\delta_W\|^2 + \|\delta_{\mathrm{BN}}\|^2 = \|\delta_W + \delta_{\mathrm{BN}}\|^2$ as a quantitative attribution principle. Eq.~\eqref{eq:decomp} is exact at the level of LP accuracies (an algebraic identity), not at the level of squared norms. The negative correlation in BN-corrupted methods (gradient ascent shifts BN running statistics toward forget while weight updates shift representations away from forget) is consistent with the two-mechanisms framing of \S\ref{sec:phenomena} but rules out additive squared-energy decomposition.

Together, weight invariance and the decomposition isolate two questions: how much of the measured leakage depends on normalization state, and how much remains after that state is recalibrated. We now examine these components across unlearning methods.

\section{Empirical Evaluation}
\label{sec:eval}

\paragraph{Setup.}
We evaluate $9$ unlearning methods on CIFAR-10 (single-class, two-class) and CIFAR-100 (single-class, five-class), with ResNet-18 (CIFAR adaptation: $3\times 3$ stem, no max-pool). Each unlearning method uses a single set of hyperparameters per dataset/forget-fraction condition, drawn from the reference implementation of each method or its closest available re-implementation; full configurations are in Appendix~\ref{app:hyperparams}. Linear probes use \texttt{sklearn.LogisticRegression(solver=`lbfgs', max\_iter=1000)} with $50$ samples per class. The probe-budget choice is biased: Adam-trained probes underestimate LP-50 by $14$--$22$pp for methods that genuinely modified the encoder (Appendix~\ref{app:lp-budget}), so we use lbfgs throughout. Relearning AUC is computed by fine-tuning on $50$ forget-class samples for $50$ epochs and integrating the normalized forget-accuracy curve. BN recalibration uses $N_r \geq 5000$ retain samples in train mode with gradients disabled. All multi-seed numbers are mean$\pm$std across three random seeds.

\subsection{Main diagnostic and self-consistency}
\label{sec:eval-main}

Across the ten rows of Table~\ref{tab:main-cifar10}, including Retrain, the association between pre-recalibration forget accuracy and leakage depends on the metric and test (Appendix~\ref{app:correlations}). Rank association with relearning is strong, but near-zero forget accuracy still leaves the relevant operating points unresolved. Among the four unlearning methods and Retrain with Pre-F $\leq 3\%$, relearning AUC spans $0.000$--$0.723$ and LP-50 spans $75.5\%$--$91.0\%$, against a retrain reference of $75.5\%$. Thus near-zero forget accuracy does not separate low-leakage checkpoints from residual leakage. After recalibration, association with relearning is $r=0.998$ (Pearson, $p<10^{-6}$), $\rho_{\mathrm{S}}=0.988$ (Spearman), and $\tau_{\mathrm{K}}=0.952$ (Kendall). Association with LP-50 is significant under all three tests ($p<0.005$).

Table~\ref{tab:main-cifar10} reports the diagnostic results (three seeds; mean$\pm$std). Six of nine methods exhibit $\Delta F \geq 18.9$pp, with the largest reversal of $+78$pp on BadTeacher: forget-class accuracy rises from $19\%$ to $97\%$ after a single weight-preserving forward pass.

\begin{table}[t]
\caption{Main diagnostic on CIFAR-10 (ResNet-18, single-class forget, class 0, three seeds; mean$\pm$std). $\Delta F = \mathrm{Post\text{-}F\%} - \mathrm{Pre\text{-}F\%}$ is the BN illusion magnitude, attributable to BN state by Corollary~\ref{cor:t2}. Pre-R / Post-R = retain accuracy before / after recalibration. PGU$^\dagger$ is our cleaner-covariance re-implementation; the official two-phase protocol yields $\Delta F = +14.5$pp (Appendix~\ref{app:pgu}).}
\label{tab:main-cifar10}
\centering
\small
\setlength{\tabcolsep}{3pt}
\resizebox{\linewidth}{!}{%
\begin{tabular}{lccccccccc}
\toprule
Method & BN & Ret-aw. & Pre-F\% & Post-F\% & $\Delta F$ & Pre-R\% & Post-R\% & LP-50 & ReAUC \\
\midrule
Retrain          & train & \checkmark & $0.0\pm0.0$  & $0.0\pm0.0$   & $0.0\pm0.0$  & $99.6$ & $99.6$ & $75.5\pm0.7$ & $0.000$ \\
SCRUB            & train & \checkmark & $0.0\pm0.0$  & $0.0\pm0.0$   & $0.0\pm0.0$  & $88.2$ & $90.9$ & $82.1\pm0.0$ & $0.000$ \\
GA               & train & $-$        & $0.0\pm0.0$  & $0.0\pm0.0$   & $0.0\pm0.0$  & $57.2$ & $81.5$ & $79.4\pm0.2$ & $0.000$ \\
GA+FT            & train & \checkmark & $67.9\pm0.5$ & $67.7\pm0.3$  & $-0.2\pm0.2$ & $96.3$ & $96.3$ & $89.7\pm0.0$ & $0.746$ \\
SalUn            & train & $-$        & $3.0\pm0.1$  & $65.1\pm0.3$  & $\mathbf{+62.1\pm0.3}$ & $85.0$ & $95.1$ & $88.9\pm0.1$ & $0.723$ \\
BadTeacher       & eval  & \checkmark & $18.9\pm6.7$ & $97.2\pm0.1$  & $\mathbf{+78.2\pm6.7}$ & $95.7$ & $95.9$ & $91.4\pm0.1$ & $0.976$ \\
IncompTeacher    & train & \checkmark & $39.1\pm22.6$& $87.1\pm11.7$ & $\mathbf{+48.0\pm11.4}$ & $90.1$ & $89.1$ & $89.8\pm0.4$ & $0.928$ \\
NoiseInject$^*$ & train & $-$        & $79.8\pm0.5$ & $98.7\pm0.0$  & $\mathbf{+18.9\pm0.5}$ & $95.0$ & $95.2$ & $91.1\pm0.0$ & $0.988$ \\
SSD              & train & $-$        & $25.7\pm3.7$ & $88.1\pm1.0$  & $\mathbf{+62.4\pm4.2}$ & $79.1$ & $95.7$ & $90.3\pm0.3$ & $0.901$ \\
PGU$^\dagger$    & eval  & $-$        & $1.1\pm0.0$  & $62.0\pm0.7$  & $\mathbf{+60.9\pm0.7}$ & $93.0$ & $94.6$ & $91.0\pm0.0$ & $0.641$ \\
\bottomrule
\end{tabular}
}\\[2pt]
{\footnotesize $^*$NoiseInject is a simplified noise-injection baseline in the style of~\citet{chundawat2023can}, not the Fisher-forgetting method of~\citet{golatkar2020eternal}. Ret-aw. = retain-aware objective or training stage; ReAUC = relearning AUC.}
\end{table}

The four rows with $|\Delta F| \leq 0.2$pp split three ways. SCRUB reaches $\Delta F=0$ via genuine erasure with retain preserved at $88\%$. GA reaches $\Delta F=0$ at $500$ steps by destroying retain (Pre-R $=57.2\%$, the lowest in the table); at fewer steps, $\Delta F$ traces a smooth phase transition from $+76$pp at $100$ steps to $0$pp at $500$ steps (Appendix~\ref{app:ga-sweep}). GA+FT's fine-tune step self-recalibrates BN. PGU (despite eval mode) exhibits $\Delta F \approx 61$pp under our re-implementation: projection controls \emph{which} directions weights move, not the \emph{magnitude} of activation drift relative to frozen running statistics, and once weights move enough to depress Pre-F\% to near zero, BN misalignment is fully exposed (Appendix~\ref{app:pgu}). \textbf{Pre-R\% reveals which $\Delta F = 0$ rows correspond to viable forgetting:} Retrain, SCRUB, and GA+FT preserve retain; GA does not, a distinction invisible from $\Delta F$ alone.

At fixed batch composition, a $3\times 3$ learning-rate/step-count grid on BadTeacher and SalUn (Appendix~\ref{app:sweep}) gives $\Delta F\geq57.8$pp for every configuration preserving retain accuracy at $\geq85\%$. The GA step-count sweep (Appendix~\ref{app:ga-sweep}) likewise links smaller gaps to degraded retain utility. Across the main comparison, the diagnostic reverses apparent forgetting in six of nine methods by $19$--$78$ percentage points; retain accuracy and the auxiliary leakage measures identify the different zero-gap regimes.

\subsection{2D decomposition: BN- vs.\ NC-residual}
\label{sec:eval-2d}

Table~\ref{tab:decomp-cifar} instantiates Eq.~\eqref{eq:decomp} on both datasets. The BN-residual is the pre-minus-post LP difference, $-\Delta\mathrm{LP}$, with weights fixed; its sign indicates whether BN measurement bias inflates ($+$) or deflates ($-$) the surface LP. The NC-residual is the post-recalibration LP elevation over retrain and is attributable to the encoder.

\begin{table}[t]
\caption{2D LP-elevation decomposition (single seed): NC-residual = $\mathrm{LP}^{M_{\mathrm{cal}}} - \mathrm{LP}^{M_{\mathrm{retr}}}$, BN-residual = $\mathrm{LP}^M - \mathrm{LP}^{M_{\mathrm{cal}}}$. CIFAR-10: $\mathrm{LP}^{M_{\mathrm{retr}}} = 76.10$. CIFAR-100: $\mathrm{LP}^{M_{\mathrm{retr}}} = 73.33$. The CIFAR-100 column for GA, SalUn, and SSD shows large \emph{negative} BN-residuals: BN deflates the pre-recalibration LP, masking the substantial positive NC-residual. Separately rounded entries can differ by $0.01$pp when subtracted.}
\label{tab:decomp-cifar}
\centering
\small
\begin{tabular}{lcccc cccc}
\toprule
& \multicolumn{4}{c}{CIFAR-10} & \multicolumn{4}{c}{CIFAR-100} \\
\cmidrule(lr){2-5} \cmidrule(lr){6-9}
Method & $\mathrm{LP}^M$ & $\mathrm{LP}^{M_{\mathrm{cal}}}$ & NC-res & BN-res & $\mathrm{LP}^M$ & $\mathrm{LP}^{M_{\mathrm{cal}}}$ & NC-res & BN-res \\
\midrule
Retrain         & 76.10 & 76.10 & 0.00  & 0.00  & 73.33 & 73.33 & 0.00  & 0.00  \\
SCRUB           & 79.23 & 79.30 & $+3.20$ & $-0.07$ & 81.67 & 86.67 & $+13.33$ & $-5.00$ \\
GA              & 67.67 & 77.93 & $+1.83$ & $-10.27$ & 58.00 & 87.67 & $+14.33$ & $-29.67$ \\
SalUn           & 77.47 & 86.23 & $+10.13$ & $-8.77$ & 60.67 & 88.00 & $+14.67$ & $-27.33$ \\
SSD             & 82.33 & 88.37 & $+12.27$ & $-6.03$ & 70.00 & 90.67 & $+17.33$ & $-20.67$ \\
GA+FT           & 89.60 & 87.30 & $+11.20$ & $+2.30$ & 88.67 & 88.67 & $+15.33$ & $0.00$ \\
BadTeacher      & 89.27 & 88.90 & $+12.80$ & $+0.37$ & 82.00 & 89.00 & $+15.67$ & $-7.00$ \\
NoiseInject    & 89.97 & 89.07 & $+12.97$ & $+0.90$ & 83.67 & 86.33 & $+13.00$ & $-2.67$ \\
IncompTeacher   & 90.63 & 87.57 & $+11.47$ & $+3.07$ & 87.00 & 90.00 & $+16.67$ & $-3.00$ \\
PGU             & 89.10 & 88.33 & $+12.23$ & $+0.77$ & 93.67 & 94.33 & $+21.00$ & $-0.66$ \\
\bottomrule
\end{tabular}
\end{table}

The empirical content of Corollary~\ref{cor:nc} is now visible. GA, SalUn, and SSD on CIFAR-100 sit far in the BN-deflation quadrant: $|\mathrm{BN\text{-}residual}| \approx 20$--$30$pp, dwarfing the $|\mathrm{NC\text{-}residual}| \approx 14$--$17$pp. The pre-recalibration LP for these methods sits \emph{below} the retrain reference; a reader looking only at the surface LP would conclude these methods stripped forget-class linear separability below the level of a model that never saw the class. The post-recalibration LP exceeds retrain by $14$--$17$pp, indicating the encoder retains substantially more forget-class structure than retrain, hidden by BN measurement bias from concurrent work~\citep{gao2026illusion}. The Neural Collapse story is therefore strengthened, not refuted: encoder-level severity is larger than concurrent work could measure, but requires recalibration to expose. For methods with small BN-residual on CIFAR-10 (BadTeacher, NoiseInject, IncompTeacher, PGU), the NC-residual is the dominant LP-elevation component and the surface LP only slightly overstates it; the qualitative attribution to encoder geometry is unchanged.

\paragraph{Quadrant geometry.}
The $9$ methods cluster into four regions of the (BN-residual, NC-residual) plane. \textbf{(I) Origin (genuine forgetting):} only Retrain occupies this region with both residuals at zero. \textbf{(II) Positive BN-residual, positive NC-residual:} CIFAR-10 BadTeacher, NoiseInject, IncompTeacher, GA+FT, and PGU; the surface LP overstates encoder leakage but is qualitatively in the right direction. \textbf{(III) Negative BN-residual, positive NC-residual (BN-deflation):} CIFAR-100 GA, SalUn, SSD; the surface LP \emph{understates} encoder leakage, sometimes by a factor that reverses the sign relative to retrain. This is the regime in which the BN illusion is most dangerous to interpret, because the surface measurement looks safer than retrain. \textbf{(IV) Sign-flipped:} no method we tested falls here, but the framework predicts a method could in principle have positive BN-residual and negative NC-residual (the surface LP overstates encoder leakage, which is itself below retrain). The four-region picture is what the unique decomposition of Proposition~\ref{prop:t3} actually buys: a rigorous taxonomy of how surface-level LP departs from encoder-level LP, applicable to any future method evaluated on a BN backbone.

The decomposition therefore distinguishes normalization that masks leakage from normalization that inflates it, while the post-recalibration residual measures elevation relative to retrain.

\paragraph{Cross-dataset replication on CIFAR-100.}
We write $\Delta\mathrm{LP}$ and $\Delta\mathrm{NCC}$ for the post-minus-pre changes in linear-probe and nearest-class-mean accuracy. We replicate the diagnostic on CIFAR-100 in single-class ($p=0.01$) and five-class ($p=0.05$) forget conditions (Appendix~\ref{app:cifar100}). Two findings. First, $|\Delta\mathrm{LP}|$ scales inversely with forget fraction $p$: for GA, CIFAR-100 1-class $\Delta\mathrm{LP}{=}{+}29.7$pp drops to $+2.1$pp at 5-class; CIFAR-10 1-class ($p{=}0.10$) is $+11.9$pp, 2-class ($p{=}0.20$) is $+5.0$pp, consistent with EMA absorption arithmetic. Second, only Retrain on CIFAR-100 5-class achieves negative $\Delta\mathrm{LP}$ ($-6.0$pp) and negative $\Delta\mathrm{NCC}$ ($-6.9$pp). Negative $\Delta\mathrm{LP}$ post-recalibration is consistent with genuine forgetting, but it is a coarse indicator and not a reliable diagnostic on its own: SCRUB, which exhibits genuine forgetting by every other measure, produces $\Delta\mathrm{LP} > 0$ in three of four CIFAR conditions we test.

\subsection{Threat model}
\label{sec:threat}
\label{sec:threat-model}

The adversary's objective is to restore forget-class capability, not to reconstruct training examples. The adversary holds the released checkpoint (weights and BN running statistics), has at least ten unlabeled natural images, and can run the model locally. The attack requires neither labels nor gradient updates nor access to the original model or designated forget set. We measure success by post-recalibration forget-class accuracy and its change from the released checkpoint; these are the quantities reported in Section~\ref{sec:eval-adv}. Membership inference is evaluated separately in Appendix~\ref{app:mia}.

\begin{center}
\small
\setlength{\tabcolsep}{4pt}
\renewcommand{\arraystretch}{1.12}
\begin{tabular}{@{}p{0.17\linewidth}p{0.25\linewidth}p{0.29\linewidth}p{0.21\linewidth}@{}}
\toprule
Setting & Actor's access & Recalibration & Relation to model release \\
\midrule
Model release & Checkpoint and local execution & Directly applies the recovery pass & Primary attack setting \\
Hosted API & Inference queries only & Cannot update BN statistics through the interface & Disclosure of the checkpoint enables the primary attack \\
Auditor / regulator & Checkpoint and retain data & Runs the same pass as an evaluation control & Same operation, different purpose \\
\bottomrule
\end{tabular}
\end{center}
The recovery procedure acts directly on a locally accessible checkpoint. An inference-only interface blocks this procedure, while an auditor uses it to test whether low forget accuracy survives normalization-state correction.

\subsection{Adversarial recovery: the artifact is exploitable}
\label{sec:eval-adv}

The diagnostic results so far are evaluation-side. The same operation, run by an attacker in the model-release setting, becomes a recovery procedure. We vary recalibration data along two axes: (i) distribution: original retain set (auditor view), CIFAR-10 test (in-distribution adversary), or CIFAR-100 raw pixels with CIFAR-10 normalization (out-of-distribution adversary); (ii) budget $B \in \{10, 25, 50, 100, 500, 1000, 2500, 5000, 9000\}$ with three random seeds. The attacker's procedure is identical to the diagnostic: a single forward pass with running statistics updated and weights frozen. No labels, no gradients, no model queries beyond the forward sweep.

\begin{figure}[t]
\centering
\includegraphics[width=1.0\linewidth]{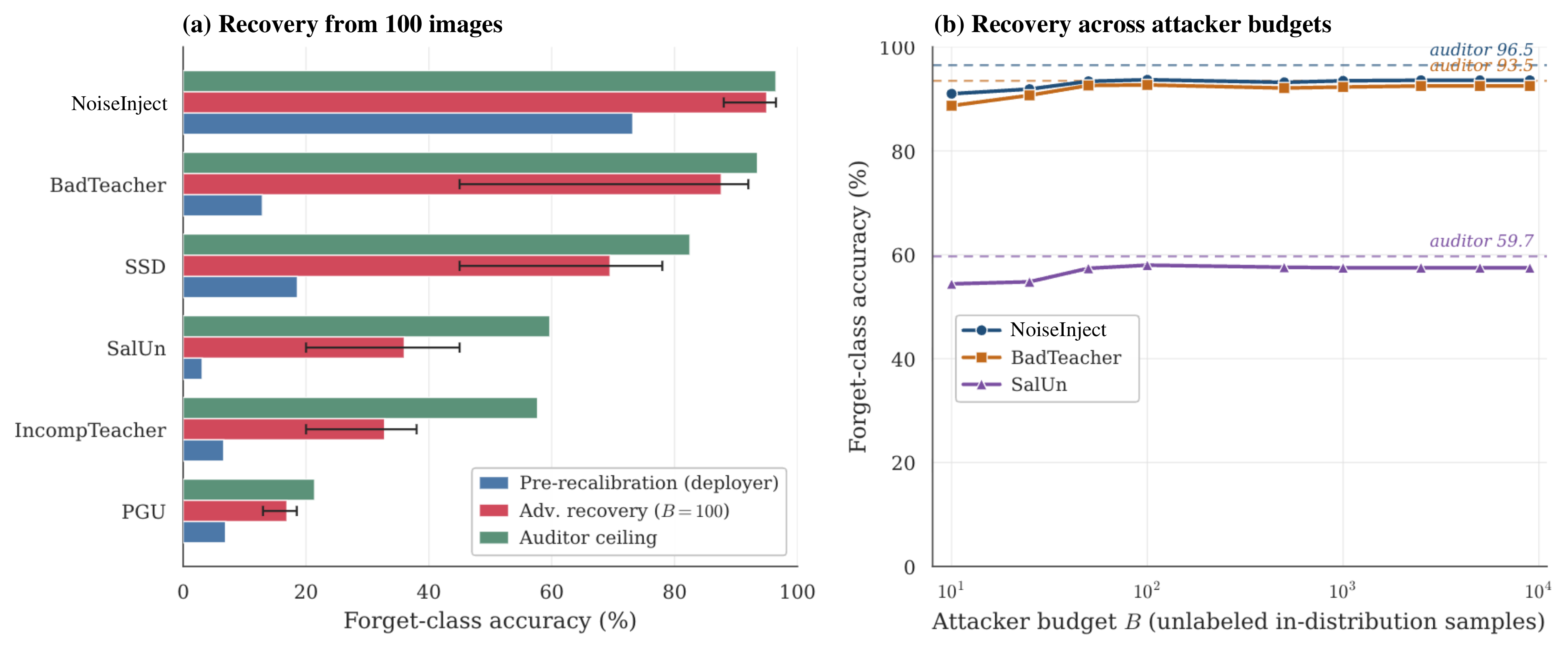}
\caption{Adversarial recovery on BN-vulnerable methods (CIFAR-10, ResNet-18). \textbf{(a)} Forget-class accuracy before recalibration (blue), after recovery using $B=100$ in-distribution unlabeled images (red; three-seed mean and seed-range error bars), and after full-retain auditor recalibration (green). \textbf{(b)} Forget-class accuracy versus attacker budget $B$ for three representative methods; dashed horizontal lines mark their auditor references. Substantial recovery is already visible at $B=10$.}
\label{fig:adv-recovery}
\end{figure}

\paragraph{Findings.}
At $B=100$, the in-distribution attack substantially restores forget-class accuracy across the vulnerable methods in Figure~\ref{fig:adv-recovery}a. The sub-100-budget results in Appendix~\ref{app:sub100} show that BadTeacher and NoiseInject reach $88.7\%$ and $91.0\%$ forget-class accuracy at $B=10$. SalUn reaches $54.4\%$ at $B=10$ and $58.0\%$ at $B=100$. These are absolute accuracies, rather than fractions of the original model's accuracy. At $B=100$, recovery with CIFAR-100 images under CIFAR-10 normalization reaches $78\%$--$95\%$ of the corresponding in-distribution accuracy (Appendix~\ref{app:ood-recovery}). Thus a ten-image budget already exposes the artifact in the tested model-release setting.

\subsection{Mechanistic falsification and scope}
\label{sec:eval-gn}

GroupNorm has no running statistics, so the recalibration analogue must be a no-op. We retrain ResNet-18 on CIFAR-10 with all \texttt{BatchNorm2d} replaced by \texttt{GroupNorm}, and re-run eight unlearning methods plus Retrain (PGU omitted on GN due to compute). $\Delta F$ collapses to exactly $0.00$pp uniformly (Appendix~\ref{app:hyperparams}, Table~\ref{tab:gn-control}), falsifying any architecture-, dataset-, or hyperparameter-level confound. Relearn-AUC remains uniformly high ($95\%$--$98\%$) on GN-ResNet; the measurement artifact disappears, but encoder-level failure~\citep{gao2026illusion,jeon2026erase} is normalization-independent. The artifact also transfers to ResNet-50 on Tiny-ImageNet: BadTeacher $\Delta F{=}{+}82.2$pp (vs.\ $+78.2$ on CIFAR-10), PGU $+29.6$pp, with the GA step-count phase transition replicating with the cliff edge shifted earlier by 200-class competitive pressure (Appendix~\ref{app:scale}). On ViT-S/16 (LayerNorm) the recalibration operator is the identity by Theorem~\ref{thm:t2}; we verify $\Delta F{=}0.00$pp across $5$ methods tuned into a forgetting regime (Appendix~\ref{app:vit}), confirming the artifact is normalization-specific. At ImageNet resolution, a pretrained ResNet-50 on ImageNet-100 gives BadTeacher $\Delta F=+92.0$pp with retain accuracy preserved, and ten-image recovery reaches $82.0\%$ forget accuracy (Appendix~\ref{app:imagenet}). Together, the controls isolate the recalibratable BN state, while the larger-image experiments demonstrate that the recovery effect extends beyond CIFAR.

\subsection{Discussion}
\label{sec:discussion}

\paragraph{Tuning at fixed batch composition.} In the tested $3\times3$ learning-rate/step-count grid, every BadTeacher or SalUn configuration preserving retain accuracy at $\geq85\%$ retains a recalibration gap of at least $57.8$pp (Appendix~\ref{app:sweep}). This result concerns that grid and fixed batch composition. Changing the retain fraction is a separate intervention: retain-mixed batches mitigate SalUn's corruption component in Appendix~\ref{app:batch-mixing}.

\paragraph{Why methods resist, and the corresponding defenses.} A small $\Delta F$ has different causes. SCRUB combines a zero gap with low relearning susceptibility and $88\%$ retain accuracy. GA reaches a zero gap at 500 steps with substantially degraded retain accuracy; Appendix~\ref{app:ga-sweep} shows the transition within the same method. GA+FT instead self-recalibrates during retain fine-tuning, giving $\Delta F=-0.2$pp. This suggests three controls. Architecturally, GN/LN remove the running statistics responsible for the artifact. Procedurally, freezing BN during unlearning prevents corruption, while retain recalibration corrects any remaining weight--normalization misalignment before evaluation. At the data level, a retain fraction of at least $25\%$ removes the observed SalUn gap in the tested mixed-batch sweep. This mixing result is specific to the evaluated pipeline; BadTeacher's eval-mode misalignment persists at its reference mixing ratio.

\paragraph{Three-regime decomposition of $\Delta F = 0$ at scale.} At CIFAR-10 scale, $\Delta F = 0$ admitted a clean reading. At Tiny-ImageNet scale (Appendix~\ref{app:scale}) it splits three ways: \emph{genuine forgetting with retain preserved} (SCRUB; high breakthrough epoch in relearning); \emph{encoder-level failure} where the head suppresses the forget class but the encoder does not move (GA+FT, SSD: high LP, high relearn-AUC); and \emph{saturated weight erasure with destroyed retain} (GA, IncompTeacher: $\leq 3\%$ retain). The auxiliary columns (Pre-R, LP, ReAUC) are required to distinguish these.

\paragraph{Practical recommendations and scope.} Deployers with retain access should recalibrate post-unlearning and serve the recalibrated model; without retain access, any unlabeled natural-image proxy fixes most of the artifact. New methods should prefer GN/LN architectures or freeze BN and recalibrate before evaluation; benchmarks should require pre/post-recalibration reporting alongside encoder-level checks~\citep{gao2026illusion,jeon2026erase}. Post-recalibration $F\%{=}0$ is not a privacy guarantee; pairing the diagnostic with MIA, MIA-NN, or UMA-style verifiers yields a more trustworthy evaluation.

\paragraph{Composition with concurrent encoder-level diagnostics.} The decomposition in Section~\ref{sec:diagnostic} positions BN recalibration as a prerequisite for, not a competitor to, encoder-level evaluation. Linear probing~\citep{gao2026illusion}, nearest-class-mean (NCC) accuracy, and backbone-freezing classifier retraining~\citep{jeon2026erase} all measure properties of the encoder, properties that, on BN-based architectures, are obscured by a measurement bias they were not designed to control for. Reporting $\mathrm{LP}^M$ alone risks the kind of sign-error illustrated in Table~\ref{tab:decomp-cifar} for GA/SalUn/SSD on CIFAR-100, where the surface measurement places these methods \emph{below} retrain on linear separability while the corrected measurement places them above. The remedy is to report $\mathrm{LP}^{M_{\mathrm{cal}}}$ or, when the BN-residual is itself of interest, both. A small additional cost (one forward sweep) buys a quantitative reading of an effect that is otherwise invisible.

\paragraph{What the diagnostic does and does not establish.} The diagnostic does not establish that any of the methods we evaluate are wrong, or that their authors made an error in their own evaluation protocols; it establishes that the standard surface-metric reading of forget accuracy on BN architectures is unreliable, and that the apparent forgetting in six of nine methods is recoverable by a transformation that modifies no weight. Whether a given method ``forgets'' is a question about the encoder; the diagnostic reframes it as a question about the recalibrated checkpoint, $M_{\mathrm{cal}}$, which is the unique representative of the model's residual forget-class capability under the unique decomposition of Proposition~\ref{prop:t3}. Honest evaluation requires reporting $\Delta F$ alongside Pre-R and Post-R. The joint value distinguishes genuine forgetting (Pre-R high, $\Delta F = 0$) from saturated weight erasure (Pre-R low, $\Delta F = 0$) from the BN illusion (Pre-R high, $\Delta F$ large) in a way no single column can.

\paragraph{Limitations.} We evaluate classification benchmarks (CIFAR-10/100, Tiny-ImageNet, and ImageNet-100), primarily with class-level forget sets. Instance-wise unlearning, with forget samples drawn randomly across all classes, is evaluated in one condition (Appendix~\ref{app:instance}). Full ImageNet-1k evaluation was not feasible within the available compute budget. We do not benchmark on language models, where the analogue of BatchNorm is LayerNorm and Theorem~\ref{thm:t2} predicts the diagnostic is the identity (the ViT-S/16 result of Appendix~\ref{app:vit} confirms this for vision LayerNorm; we do not test text). The threat model assumes local checkpoint access; defending against an attacker with checkpoint access requires architectural changes (GN/LN) rather than evaluation-side controls.

\paragraph{Why the artifact has gone undocumented.} The recalibration operation is standard practice in test-time adaptation~\citep{schneider2020improving,nado2020evaluating} and SWA~\citep{izmailov2018averaging}; the mechanism is well-known. We attribute the absence of recalibration as an unlearning evaluation control to three factors. First, retain-set BN recalibration is not commonly listed as a step in the standard unlearning pipeline; authors evaluate on the model produced by the method, exactly as released. Second, the artifact is direction-specific: BN corruption \emph{lowers} forget accuracy on the surface, which looks like the desired outcome rather than the opposite of it. Third, the most-cited unlearning benchmarks~\citep{triantafillou2024unlearning,cadet2024deepunlearn,huang2025mubox} score methods on metrics that, on BN backbones, the standard recalibration would shift; in the absence of an explicit control, one cannot tell whether a high-scoring method is genuinely forgetting or successfully exploiting the artifact. Our diagnostic adds a single line of code per evaluated checkpoint and recovers the missing control. The cost is one forward pass with gradients disabled; the value is a quantitative, weight-preserving, mechanistically grounded reading of an effect that previously had no name.

\section{Conclusion}
\label{sec:conclusion}

A single forward pass over retain data, modifying no weight, reverses the apparent forgetting in six of nine evaluated unlearning methods, is exploitable by an attacker with as few as ten unlabeled images, and is absent under GroupNorm. The operation is a deterministic, weight-preserving fixed-point on the model's normalization state, so any pre/post gap is provably attributable to BatchNorm rather than to weights. Headline forgetting numbers on BN architectures should be re-reported with the diagnostic applied; the artifact is mechanistically distinct from, and complementary to, the encoder-level failures concurrent work has documented, with our decomposition revealing that BN measurement bias \emph{masks} encoder leakage rather than creating it for several methods. Both normalization-state and encoder-level controls are therefore needed. The membership-inference experiment reinforces the metric-specific scope: recalibration changes attack AUC by at most $0.012$, despite forget-accuracy changes of up to approximately $98$pp on those checkpoints (Appendix~\ref{app:mia}). The diagnostic is one function call. We name the effect the \emph{BN illusion}. Future BN-based unlearning evaluations should report $\Delta F$ alongside Pre-R and Post-R as standard practice; the recalibrated checkpoint $M_{\mathrm{cal}}$ is the unbiased reading of the model's residual forget-class capability.

\bibliographystyle{plainnat}
\bibliography{references}

\newpage
\appendix

\section{BN Corruption: Formal Statement and Proof}
\label{app:t1}

In training mode, batch statistics update the running buffers during the forward pass. An operation applied subsequently to weight gradients does not intercept this update. The quantitative effect depends on the separation of the activation moments, rather than on class labels alone.

\begin{theorem}[BN corruption / non-interceptability]
\label{thm:t1}
Let $u_t$ be the running mean at a BN layer. Assume \textbf{(A1)} that the forget and retain pre-normalization mean vectors at the weights under consideration satisfy
\begin{equation}
 d_\ell := \frac{\|\mu_f^\ell-\mu_r^\ell\|_2}{\|\sigma_r^\ell\|_2}
 \geq \delta>0, \qquad \|\sigma_r^\ell\|_2>0,
 \label{eq:activation-separation}
\end{equation}
and \textbf{(A2)} that BN runs in training mode with EMA momentum $m\in(0,1]$. Every unlearning-batch forward pass executes
$u_{t+1}=(1-m)u_t+m\hat\mu_t$, independently of any operator applied only to weight gradients. If, additionally, the activation-moment targets remain fixed, the batches have expected mean $\mu_f^\ell$, and $u_0=\mu_r^\ell$, then
\begin{equation}
 \frac{\|\mathbb E u_T-\mu_r^\ell\|_2}{\|\sigma_r^\ell\|_2}
 = [1-(1-m)^T]d_\ell
 \geq [1-(1-m)^T]\delta.
 \label{eq:stationary-ema}
\end{equation}
This bound concerns normalization-state drift, not a lower bound on forget accuracy.
\end{theorem}

\begin{proof}
Unrolling the forward-pass update gives
\begin{equation}
 u_T=(1-m)^Tu_0+m\sum_{t=0}^{T-1}(1-m)^{T-1-t}\hat\mu_t.
 \label{eq:ema-unroll}
\end{equation}
Under the fixed-target assumptions, take expectations and use
$m\sum_{t=0}^{T-1}(1-m)^{T-1-t}=1-(1-m)^T$ to obtain
$\mathbb E u_T-\mu_r^\ell=[1-(1-m)^T](\mu_f^\ell-\mu_r^\ell)$.
Equation~\eqref{eq:stationary-ema} follows from (A1). The EMA update uses the current batch statistics, not the weight gradients. A gradient projection can alter later weights and hence later activation distributions, but cannot prevent the buffer update already executed by the forward pass.
\end{proof}

\paragraph{Changing weights and moving targets.}
For an arbitrary weight trajectory, let $a_t$ be the conditional expected mean of the current unlearning batch, $q_t$ the contemporaneous retain mean, and $\xi_t=\hat\mu_t-a_t$. Writing $w_t=(1-m)^{T-1-t}$ gives the exact identity
\begin{equation}
\begin{aligned}
 u_T-q_T={}&(1-m)^T(u_0-q_0)
       +m\sum_{t=0}^{T-1}w_t(a_t-q_t)\\
       &+\sum_{t=0}^{T-1}w_t(q_t-q_{t+1})
       +m\sum_{t=0}^{T-1}w_t\xi_t.
\end{aligned}
\label{eq:moving-ema}
\end{equation}
The terms separate initial mismatch, activation separation, movement of the retain target, and batch noise. A cumulative lower bound for changing weights requires control of target motion and cancellation between the separation vectors. The forward-pass non-interceptability statement itself holds throughout the trajectory. Appendix~\ref{app:divergence} measures activation separation at the pre-unlearning weights.

\paragraph{Mixed batches and loss independence.}
For a mixed loader, $a_t$ in Eq.~\eqref{eq:moving-ema} is the actual pre-BN mean induced by that loader. When the feature map before the layer is fixed independently of batch composition, a retain fraction $r_{\mathrm{mix}}$ gives
$a_{\mathrm{mix}}=(1-r_{\mathrm{mix}})\mu_f+r_{\mathrm{mix}}\mu_r$, reducing the mean separation by the factor $1-r_{\mathrm{mix}}$. Deeper train-mode BN layers can themselves change this feature map through earlier batch normalizations. The mixed-batch experiment in Appendix~\ref{app:batch-mixing} therefore measures the resulting effect directly. The update does not depend on which loss generated the gradients: the structural statement applies to gradient-ascent and data-based procedures alike, including RandomLabeling (Appendix~\ref{app:random-label}).

\paragraph{Limits of the static-target quantitative prediction.}
A quasi-stationary approximation predicts that relative excess BN drift, normalized by the natural retain offset, scales as
$[1-(1-m)^T](1-p)/p$ when weight-induced activation drift is small. Across five CIFAR-10/100 conditions, the fitted slope relating this prediction to measured drift was $0.019$, compared with the identity-prediction slope of $1.0$. The approximation did not provide a useful quantitative prediction: during unlearning, weight updates change the activation moments that the EMA tracks, as Eq.~\eqref{eq:moving-ema} makes explicit. We retain the structural mechanism but do not use this static-target approximation to predict drift magnitudes. Recalibration in Theorem~\ref{thm:t2} is applied after unlearning, at fixed weights, so its idempotence concerns a different operation.

\section{Recalibration: Proof, Finite-Batch Error, and Layerwise Measurements}
\label{app:t2-proof}

\subsection{Population properties and empirical idempotence}
\label{app:population-proof}

\begin{proof}[Proof of Theorem~\ref{thm:t2}(i)--(iii)]
\textbf{(i)} The state $\varphi^*(\theta;\mathcal D_r)$ depends only on the weights and retain distribution. Since $\mathcal R^*$ leaves $\theta$ unchanged, a second application returns the same state.
\textbf{(ii)} Weight invariance follows from Eq.~\eqref{eq:recalop}. Pre-BN activations at the first layer are unchanged. At later layers they depend on the updated normalization of earlier layers and may change; the attribution uses weight invariance, not activation invariance.
\textbf{(iii)} At each BN layer, eval-mode normalization with its oracle moments computes
$y_c=\gamma_c(h_c-\mu_c^*)/\sqrt{v_c^*+\epsilon}+\beta_c$.
Starting from the first layer and propagating through the network, this is exactly the oracle-normalized forward map at the fixed weights.
\end{proof}

The empirical pass resets the running statistics and uses train-mode batch statistics for forward normalization. For a fixed sequence of input tensors and batches $\mathcal B$, deterministic computation therefore produces a state $\widehat\varphi_{\mathcal B}(\theta)$ independent of the incoming buffers:
\begin{equation}
 \widehat{\mathcal R}_{\mathcal B}(\theta,\varphi)
  =(\theta,\widehat\varphi_{\mathcal B}(\theta)),
 \qquad
 \widehat{\mathcal R}_{\mathcal B}\circ\widehat{\mathcal R}_{\mathcal B}
  =\widehat{\mathcal R}_{\mathcal B}.
 \label{eq:empirical-idempotence}
\end{equation}
Repeated passes with the same inputs and partition yielded bit-identical running statistics. Resampling the partition changes the empirical operator; the observed variation was approximately $10^{-2}$. Stochastic input transformations and any other forward-pass randomness must also be fixed for exact repeatability.

\subsection{Finite-batch derivation}
\label{app:finite-batch}

\paragraph{Assumptions.}
Hold the weights and finite network depth $L$ fixed. Let $P=\mathcal D_r$ be a distribution of independent input images, with sufficient moments for the mean and variance estimators below; spatial sites within an image need not be independent. BN uses a fixed $\epsilon>0$. Assume the recursively defined activation-moment functionals admit the second-order expansion in Eq.~\eqref{eq:moment-expansion}, with the stated moment and remainder bounds. For ReLU networks this is regularity of the expected moment functionals, including activation-threshold crossings, rather than pointwise twice differentiability of ReLU. First consider $N_r=Kb$ images in $K$ equal-sized batches.

\paragraph{The local expansion and the Jensen term.}
At one scalar BN coordinate write
$F(h;u,v)=\gamma(h-u)(v+\epsilon)^{-1/2}+\beta$ and $a=v+\epsilon$.
For moment errors $\delta_u=\hat u-u$ and $\delta_v=\hat v-v$, Taylor expansion in the estimated moments gives
\begin{equation}
\begin{aligned}
 F(h;\hat u,\hat v)-F(h;u,v)
 ={}&-\frac{\gamma\delta_u}{\sqrt a}
     -\frac{\gamma(h-u)\delta_v}{2a^{3/2}}\\
    &+\frac{\gamma\delta_u\delta_v}{2a^{3/2}}
     +\frac{3\gamma(h-u)\delta_v^2}{8a^{5/2}}
     +\mathcal E_b.
\end{aligned}
\label{eq:bn-taylor}
\end{equation}
The reciprocal-standard-deviation map
$q(v)=(v+\epsilon)^{-1/2}$ is convex, since
$q''(v)=3(v+\epsilon)^{-5/2}/4>0$. Jensen's inequality applies to the random batch variance $\hat v$:
$\mathbb E q(\hat v)\geq q(\mathbb E\hat v)$.
The quadratic terms in Eq.~\eqref{eq:bn-taylor} have order $b^{-1}$ when moment fluctuations have size $b^{-1/2}$ and the corresponding moments and remainder are controlled. This does not assign a universal sign to the complete output bias: $h$, $\hat u$, and $\hat v$ come from the same batch and are dependent. The following calculation retains that dependence.

\paragraph{From one batch to the full network.}
Let $\mathcal T_\ell(P)$ be the pre-BN population mean and variance at layer $\ell$ when preceding layers use their oracle $P$ moments. It includes within-image spatial averaging. For the empirical image distribution
$P_b=b^{-1}\sum_{i=1}^{b}\delta_{X_i}$, train-mode normalization uses the corresponding batch moments $\mathcal T_\ell(P_b)$ throughout the preceding layers. Write the assumed second-order expansion as
\begin{equation}
\begin{aligned}
 \mathcal T_\ell(P_b)-\mathcal T_\ell(P)
  ={}&\frac1b\sum_{i=1}^{b}\psi_\ell(X_i)
    +\frac1{2b^2}\sum_{i,j=1}^{b}\Psi_\ell(X_i,X_j)
    +r_{\ell,b}.
\end{aligned}
\label{eq:moment-expansion}
\end{equation}
Here $\mathbb E\psi_\ell(X)=0$ and
$\mathbb E\Psi_\ell(x,X)=\mathbb E\Psi_\ell(X,x)=0$ for every $x$; the kernels have finite second moments, including on the diagonal. Assume
$\mathbb E\|r_{\ell,b}\|=o(b^{-1})$ and
$\mathbb E\|r_{\ell,b}\|^2=o(b^{-1})$.
The first-order expectation vanishes. Independence of images eliminates the off-diagonal expectations in the double sum, leaving its $b$ diagonal terms:
\begin{equation}
 \mathbb E\mathcal T_\ell(P_b)-\mathcal T_\ell(P)
 =\frac{\mathbb E\Psi_\ell(X,X)}{2b}+o(b^{-1}),
 \qquad
 \mathbb E\|\mathcal T_\ell(P_b)-\mathbb E\mathcal T_\ell(P_b)\|^2
 =O_L(b^{-1}).
 \label{eq:batch-moment-bias}
\end{equation}
This yields the finite-batch offset without treating a sample and its batch statistics as independent.

PyTorch uses the biased batch variance for forward normalization and a Bessel-corrected variance for its running-variance update.\footnote{\url{https://docs.pytorch.org/docs/2.10/generated/torch.nn.BatchNorm2d.html}} For $bS_\ell$ scalar sites per channel, the factor $bS_\ell/(bS_\ell-1)$ contributes another $O(b^{-1})$ term at fixed spatial size $S_\ell$; it does not make spatial sites independent. Thus the actual statistic $S_{\ell,b}$ written into the running average satisfies
\begin{equation}
 \mathbb E S_{\ell,b}=\varphi_\ell^*+\frac{c_\ell}{b}+o(b^{-1}),
 \qquad
 \mathbb E\|S_{\ell,b}-\mathbb E S_{\ell,b}\|^2=O_L(b^{-1}).
 \label{eq:written-stat}
\end{equation}

\paragraph{Averaging retain batches.}
For an arithmetic average of $K=N_r/b$ independent batches,
$\widehat\varphi_{\ell,N_r,b}=K^{-1}\sum_{k=1}^{K}S_{\ell,b}^{(k)}$.
Equation~\eqref{eq:written-stat} gives a mean-squared sampling fluctuation
$K^{-1}O_L(b^{-1})=O_L(N_r^{-1})$, while the expectation retains the finite-batch offset. Over the finite collection of layers and channels,
\begin{equation}
 \boxed{\|\widehat\varphi_{N_r,b}-\varphi^*\|
       =O_p(N_r^{-1/2})+O_L(b^{-1}).}
 \label{eq:full-error}
\end{equation}
For unequal batch sizes $b_k$ and normalized averaging weights $w_k$, the corresponding terms are
$O_p((\sum_k w_k^2/b_k)^{1/2})$ and
$O_L(\sum_k w_k/b_k)$, using the actual size and weight of the final batch.

The functionals $\mathcal T_\ell$ already include propagation through earlier layers, so their constants depend on depth, weights, and activation sensitivities. Equation~\eqref{eq:full-error} is a fixed-depth result, not a depth-uniform bound. Increasing $N_r$ at fixed $b$ reduces sampling variation but does not remove the finite-batch term.

\subsection{Empirical reference agreement and layerwise mismatch}
\label{app:bn-measurements}

\paragraph{Agreement between recalibration configurations.}
We compare statistics estimated from $5{,}000$ retain images with an empirical reference estimated from all $45{,}000$ retain images at batch size $256$, using \texttt{update\_bn} at the same fixed unlearned weights. At the default batch size of $128$, the layer-averaged normalized running-mean discrepancy,
\begin{equation}
 \frac1L\sum_{\ell=1}^{L}
 \frac{\|\widehat\mu_{5000,128}^{\ell}-\widehat\mu_{45000,256}^{\ell}\|_2}
      {\|\widehat\sigma_{45000,256}^{\ell}\|_2},
 \label{eq:empirical-reference-gap}
\end{equation}
is $0.0016$--$0.0018$ across GA, SalUn, SCRUB, and BadTeacher. These measurements quantify agreement with the full-retain empirical reference, which itself uses finite batches, rather than error against the oracle $\varphi^*$.

\paragraph{Layerwise normalization mismatch.}
Figure~\ref{fig:bn-depth} compares each unlearned checkpoint with its retain-recalibrated copy. Define
\begin{equation}
 \delta_\mu^\ell=
 \frac{\|\mu^\ell-\widehat\mu_{\mathrm{cal}}^\ell\|_2}
      {\|\widehat\sigma_{\mathrm{cal}}^\ell\|_2},
 \qquad
 \delta_v^\ell=
 \frac{\|(\sigma^\ell)^2-(\widehat\sigma_{\mathrm{cal}}^\ell)^2\|_2}
      {\|(\widehat\sigma_{\mathrm{cal}}^\ell)^2\|_2}.
 \label{eq:depth-mismatch}
\end{equation}
Across the 20 BN layers, the mean-mismatch profiles are method-dependent and non-monotonic. Fitted slopes per layer are $-0.0040$ for GA, $-0.0017$ for SCRUB, $+0.0018$ for SalUn, and $+0.0056$ for BadTeacher. The largest observed layerwise mean mismatch is $0.3402$, at BadTeacher's final BN layer. These profiles locate the normalization mismatch corrected by recalibration; they measure a different quantity from finite-batch estimation bias.

\begin{figure}[t]
\centering
\includegraphics[width=0.49\linewidth]{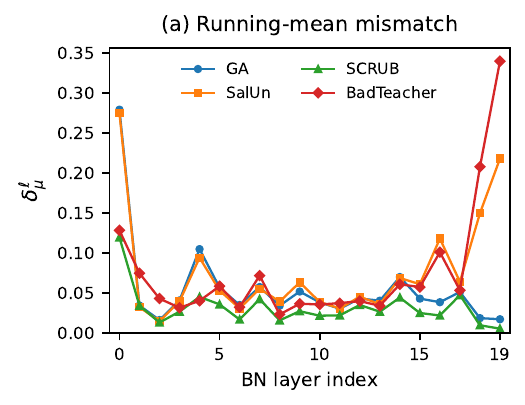}\hfill
\includegraphics[width=0.49\linewidth]{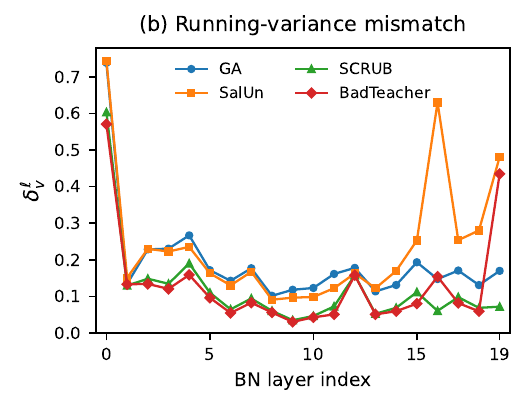}
\caption{Layerwise normalization mismatch on CIFAR-10/ResNet-18, forgetting class 0. Each checkpoint is compared with its own retain-recalibrated copy ($N_r=5000$, $b=128$, fixed subset and batch order). \textbf{(a)} Normalized running-mean mismatch. \textbf{(b)} Relative running-variance mismatch, as defined in Eq.~\eqref{eq:depth-mismatch}. GA, SalUn, and SCRUB use train-mode BN during unlearning; BadTeacher uses eval mode. All 20 BN layers are shown in module order. The experiment uses seed 0; no across-seed error bars are shown.}
\label{fig:bn-depth}
\end{figure}

\section{Proof of Proposition~\ref{prop:t3}}
\label{app:t3-proof}

\begin{proof}
The decomposition is the identity $A - C = (B - C) + (A - B)$ for $A = \mathrm{LP}^M, B = \mathrm{LP}^{M_{\mathrm{cal}}}, C = \mathrm{LP}^{M_{\mathrm{retr}}}$. Existence: when $\varphi = \varphi^*$, then $M = M_{\mathrm{cal}}$ so the BN-residual $A - B = 0$; under $\mathcal{R}^*$, $A \to B, B \to B, C \to C$, so the NC-residual $B - C$ is invariant. Uniqueness: suppose $\mathrm{LP}^M - \mathrm{LP}^{M_{\mathrm{retr}}} = X + Y$ with $X$ vanishing at $\varphi = \varphi^*$ and $Y$ invariant under $\mathcal{R}^*$. Evaluate the identity on $M_{\mathrm{cal}}$: the left side becomes $B-C$, $X$ vanishes, and $Y$ is unchanged by invariance. So $Y = B - C$, hence $X = A - B$. Corollary~\ref{cor:nc} follows: the BN-residual is induced by an operator leaving encoder weights unchanged, so any nonzero BN-residual reflects a measurement bias, not an encoder-geometry change.
\end{proof}

\section{Cross-Correlation: Why the Pythagorean Form Fails}
\label{app:cross-corr}

For each method we measured the empirical correlation $\rho$ between weight-induced ($\delta_W$) and BN-state-induced ($\delta_{\mathrm{BN}}$) penultimate-layer feature perturbations on forget-class inputs. The implementation computes
\begin{align*}
\delta_W &= h(\theta_T,\varphi^*)-h(\theta_0,\varphi^*(\theta_0)),\\
\delta_{\mathrm{BN}} &= h(\theta_T,\varphi_T)-h(\theta_T,\varphi^*).
\end{align*}
The first compares weights with both models recalibrated; the second isolates the statistics effect at fixed final weights. Here $(\theta_0, \varphi_0)$ is the pre-unlearning checkpoint, $(\theta_T, \varphi_T)$ is the post-unlearning checkpoint, and $\varphi^* = \varphi^*(\theta_T; \mathcal{D}_r)$ is the recalibrated state defined in Section~\ref{sec:operator}. Computing $\delta_{\mathrm{BN}}$ against $\varphi^*$ rather than $\varphi_0$ means $\delta_{\mathrm{BN}}$ measures the deviation of the post-unlearning running stats from the population fixed point at fixed final weights---the same quantity that drives the diagnostic $\Delta F$ in the main paper.

\paragraph{Regime-dependence.} For methods that update BN running statistics during unlearning (train-mode methods: GA, SCRUB, SalUn, IncompTeacher, NoiseInject, SSD, GA+FT), $\varphi_T$ diverges from the population fixed point during unlearning, and $\delta_{\mathrm{BN}}$ captures the corruption-induced perturbation. For methods that freeze BN during unlearning (eval-mode methods: BadTeacher, and PGU in our re-implementation), running statistics never move: $\varphi_T = \varphi_0$ exactly, verified by direct checkpoint comparison (max $\|\varphi_T - \varphi_0\|_\infty = 0$ across all 20 BN layers for BadTeacher). For these methods $\delta_{\mathrm{BN}}$ reduces to $h(\theta_T, \varphi_0) - h(\theta_T, \varphi^*)$, a calibration gap between full-data normalization ($\varphi_0$, computed from 50k training samples over all classes) and retain-only normalization ($\varphi^*$, computed from 45k retain samples over 9 classes). This gap is small in LP space (BadTeacher's BN-residual in Table~\ref{tab:decomp-cifar} is $+0.37$pp on CIFAR-10) but has a fixed magnitude in feature space because both stat sets are well-calibrated to their respective distributions.

\begin{table}[h]
\caption{Cross-correlation analysis on CIFAR-10 (single-class). $E_W = \|\delta_W\|^2$, $E_{\mathrm{BN}} = \|\delta_{\mathrm{BN}}\|^2$, $E_{\mathrm{total}} = \|\delta_W + \delta_{\mathrm{BN}}\|^2$. Pyth-error $= |E_W + E_{\mathrm{BN}} - E_{\mathrm{total}}|/E_{\mathrm{total}}$. Methods grouped by BN mode during unlearning.}
\label{tab:crosscorr}
\centering
\small
\begin{tabular}{llccccc}
\toprule
Method & BN mode & $E_W$ & $E_{\mathrm{BN}}$ & $E_{\mathrm{total}}$ & $\rho$ & Pyth-error \\
\midrule
\multicolumn{7}{l}{\textit{Train-mode methods}} \\
GA            & train & 57.4 & 24.4 & 64.1 & $-0.24$ & 27.6\% \\
SCRUB         & train & 46.9 & 3.1  & 49.6 & $-0.02$ & 0.7\% \\
SalUn         & train & 50.3 & 34.2 & 46.2 & $-0.46$ & 83.2\% \\
NoiseInject  & train & 58.8 & 7.3  & 41.6 & $-0.59$ & 59.1\% \\
IncompTeacher & train & 42.6 & 4.8  & 27.9 & $-0.68$ & 69.8\% \\
GA+FT         & train & 49.7 & 1.1  & 53.4 & $+0.17$ & 4.8\% \\
SSD           & train & 47.2 & 21.5 & 27.7 & $-0.64$ & 147.6\% \\
\midrule
\multicolumn{7}{l}{\textit{Eval-mode methods}} \\
BadTeacher$^\dagger$ & eval & 42.6 & 20.1 & 24.8 & $-0.65$ & 152.6\% \\
\bottomrule
\end{tabular}\\[2pt]
{\footnotesize $^\dagger$BN frozen during unlearning; $\delta_{\mathrm{BN}}$ measures the retain-data calibration gap between $\varphi_0$ and $\varphi^*$ rather than unlearning-induced corruption (see prose).}
\end{table}

\paragraph{Headline reading.} Across the seven train-mode methods, mean $|\rho| = 0.40$ with maximum $0.68$. The threshold for a Pythagorean form ($|\rho| < 0.2$ across all rows) is not met. The Pythagorean error exceeds $50\%$ on three of seven train-mode methods and reaches $147.6\%$ on SSD. This is the load-bearing finding: an additive squared-energy decomposition $\|\delta_W\|^2 + \|\delta_{\mathrm{BN}}\|^2 = \|\delta_W + \delta_{\mathrm{BN}}\|^2$ fails empirically, which is why the LP-space decomposition of Eq.~\eqref{eq:decomp} is presented as an algebraic identity rather than as an orthogonal Pythagorean split. The negative correlation pattern in train-mode BN-corrupted methods (gradient ascent shifts running stats toward forget while weight updates shift representations away from forget) is mechanistically consistent with the two-phenomena framing of Section~\ref{sec:phenomena}---the weight-induced and stat-induced perturbations partially cancel, which is what produces the systematically negative $\rho$ values on the three large-Pyth-error rows (SalUn, NoiseInject, SSD).

The eval-mode row carries a different interpretation. BadTeacher's row reflects the calibration gap between $\varphi_0$ and $\varphi^*$, not unlearning-induced corruption (since BadTeacher's stats never moved). The large $E_{\mathrm{BN}}$ and Pyth-error for BadTeacher are not evidence of BN corruption; they are evidence that recalibrating to retain-only stats induces a feature-space perturbation whose direction is anti-correlated with the unlearning weight perturbation ($\rho = -0.65$). We retain the row in Table~\ref{tab:crosscorr} for completeness and to make this dependence explicit.

\paragraph{Why the LP-space decomposition is the load-bearing one.} The decomposition of Eq.~\eqref{eq:decomp}---BN-residual $= \mathrm{LP}^M - \mathrm{LP}^{M_{\mathrm{cal}}}$, NC-residual $= \mathrm{LP}^{M_{\mathrm{cal}}} - \mathrm{LP}^{M_{\mathrm{retr}}}$---is uniform across BN modes. For eval-mode BadTeacher, the BN-residual on CIFAR-10 is $+0.37$pp (Table~\ref{tab:decomp-cifar}), correctly reflecting that BadTeacher's surface LP measurement is barely biased by its (frozen, retain-aligned) BN stats. For eval-mode PGU on CIFAR-10, the BN-residual is $+0.77$pp (Table~\ref{tab:decomp-cifar}). For train-mode SalUn, GA, and SSD on CIFAR-100, the BN-residuals are large and negative ($-27.33$, $-29.67$, $-20.67$), reflecting that BN measurement bias deflates the surface LP for these methods. The LP-space decomposition cleanly distinguishes regime from method, while the feature-space cross-correlation in Table~\ref{tab:crosscorr} does not. This is one reason the LP-space form is the primary decomposition in the paper; the feature-space analysis is presented in this appendix solely to justify the algebraic-identity (rather than energy) framing of Eq.~\eqref{eq:decomp}.

\section{CIFAR-100 Full Results}
\label{app:cifar100}

\begin{table}[h]
\caption{CIFAR-100 single-class forget (class $0$, $p=0.01$; single seed). $\Delta\mathrm{LP}=\mathrm{LP}^{M_{\mathrm{cal}}}-\mathrm{LP}^M=-\mathrm{BN\text{-}residual}$.}
\label{tab:cifar100-1class}
\centering
\small
\begin{tabular}{lccccccc}
\toprule
Method & Post-F\% & LP-pre & LP-post & $\Delta\mathrm{LP}$ & NCC-pre & NCC-post & Relearn AUC \\
\midrule
Retrain & 0.0 & 73.3 & 74.0 & $+0.7$ & 72.3 & 69.0 & 0.633 \\
SCRUB & 0.0 & 81.7 & 86.7 & $+5.0$ & 57.3 & 78.7 & 0.306 \\
GA & 1.0 & 58.0 & 87.7 & $+29.7$ & 17.0 & 81.3 & 0.811 \\
SalUn & 1.0 & 60.7 & 88.0 & $+27.3$ & 19.0 & 83.3 & 0.815 \\
SSD & 0.0 & 70.0 & 90.7 & $+20.7$ & 49.3 & 89.1 & 0.862 \\
GA+FT & 1.0 & 88.7 & 88.7 & $0.0$ & --   & --   & 0.824 \\
IncompTeacher & 75.0 & 87.0 & 90.0 & $+3.0$ & 70.0 & 76.3 & 0.969 \\
BadTeacher & 28.0 & 75.0 & 89.0 & $+14.0$ & 47.0 & 58.2 & 0.965 \\
NoiseInject & 85.0 & 88.4 & 91.1 & $+2.7$ & 64.7 & 71.4 & 0.987 \\
PGU & 4.0 & --   & 94.3 & --      & --   & --   & 0.991 \\
\bottomrule
\end{tabular}
\end{table}

\begin{table}[h]
\caption{CIFAR-100 $5$-class forget (classes $0$--$4$, $p=0.05$; single seed).}
\label{tab:cifar100-5class}
\centering
\small
\begin{tabular}{lccccccc}
\toprule
Method & Post-F\% & LP-pre & LP-post & $\Delta\mathrm{LP}$ & NCC-pre & NCC-post & Relearn AUC \\
\midrule
Retrain & 0.0 & 54.9 & 48.9 & $-6.0$ & 51.5 & 44.6 & 0.596 \\
SCRUB & 0.0 & 65.6 & 66.6 & $+1.0$ & 59.3 & 60.7 & 0.828 \\
GA & 0.0 & 43.8 & 45.9 & $+2.1$ & 34.5 & 37.1 & 0.790 \\
SalUn & 0.0 & 42.1 & 44.4 & $+2.3$ & 32.9 & 36.1 & 0.800 \\
\bottomrule
\end{tabular}
\end{table}

\section{Hyperparameters and Implementation Details}
\label{app:hyperparams}

\paragraph{Architecture.} ResNet-18 with CIFAR adaptation: $3\times 3$ stem conv, no max-pool, $512 \to 10$ or $100$ classifier. GroupNorm variant: all \texttt{nn.BatchNorm2d(c)} replaced by \texttt{nn.GroupNorm(num\_groups=32, num\_channels=c)}.

\paragraph{Method hyperparameters (CIFAR-10 single-class).} SCRUB: lr=$10^{-3}$, steps=$500$, retain\_steps=$2$, kl\_weight=$1.0$. IncompTeacher: lr=$10^{-3}$, steps=$500$, $\tau_{\mathrm{dist}}=4.0$, $\alpha=0.95$. SalUn: lr=$10^{-3}$, steps=$200$, threshold=$0.5$. GA: lr=$10^{-3}$, steps=$500$. SSD: $\lambda=10000$. NoiseInject (simplified noise baseline): $\sigma_{\mathrm{noise}}=5.0$. GA+FT: GA-steps=$200$, FT-epochs=$1$, FT-lr=$10^{-4}$. PGU: see Appendix~\ref{app:pgu}.

\paragraph{Diagnostic protocol.} BN recalibration: \texttt{torch.optim.swa\_utils.update\_bn} on retain loader, train mode, no gradients, mini-batch size $b=128$. LP-50: \texttt{sklearn.linear\_model.LogisticRegression(solver='lbfgs', max\_iter=1000)}, $50$ samples per class. NCC: class-mean features on $50$ samples per class, $\ell_2$ assignment. Relearn-AUC: fine-tune on $50$ forget-class samples for $50$ epochs (lr=$10^{-4}$), record forget accuracy curve, integrate normalized.

\paragraph{Additional experiment cohorts.} The normalization diagnostics in Appendix~\ref{app:t2-proof} and the CIFAR experiments in Appendices~\ref{app:mia}--\ref{app:random-label} use a separately trained base model from the main evaluation. Comparisons are paired within each experiment; absolute values across cohorts are not directly comparable. Appendix~\ref{app:correlations} instead recomputes correlations from Table~\ref{tab:main-cifar10}. The ImageNet-100 experiment uses its own pretrained ResNet-50.

\paragraph{Compute.} All experiments on a single NVIDIA A6000 GPU ($\approx 30$GB usable VRAM). Total wall-clock $\approx 90$ GPU-hours.

\begin{table}[h]
\caption{GroupNorm control on CIFAR-10. BN $\Delta F$ reproduces Table~\ref{tab:main-cifar10}'s 3-seed means; GN $\Delta F$ is the same metric on a GN-ResNet retrained from scratch (single seed; mechanistic guarantee yields zero variance).}
\label{tab:gn-control}
\centering
\small
\begin{tabular}{lccc}
\toprule
Method & BN $\Delta F$ (pp) & GN Pre-F\% & GN $\Delta F$ (pp) \\
\midrule
GA & $+0.0$ & 89.9 & $0.00$ \\
SalUn & $+62.1$ & 86.4 & $0.00$ \\
BadTeacher & $+78.2$ & 0.2 & $0.00$ \\
IncompTeacher & $+48.0$ & 0.1 & $0.00$ \\
SCRUB & $+0.0$ & 0.0 & $0.00$ \\
GA+FT & $-0.2$ & 91.4 & $0.00$ \\
NoiseInject & $+18.9$ & 93.6 & $0.00$ \\
SSD & $+62.4$ & 93.6 & $0.00$ \\
Retrain & $+0.0$ & 0.0 & $0.00$ \\
\bottomrule
\end{tabular}
\end{table}

\paragraph{GN per-method retain/LP detail.} GN-ResNet retrains from scratch to $93.97\%$ test accuracy. GN retain / GN LP-50: GA $94.2/91.2$; SalUn $94.4/89.2$; BadTeacher $58.9/87.1$; IncompTeacher $93.4/95.7$; SCRUB $54.9/67.8$; GA+FT $94.3/92.0$; NoiseInject $94.0/93.5$; SSD $94.0/93.4$; Retrain $90.8/67.2$. BadTeacher and SCRUB exhibit retain degradation on \emph{both} BN- and GN-ResNet---a method-level limitation, not a normalization confound. Five of the eight GN methods (GA, SalUn, GA+FT, NoiseInject, SSD) did not reach a forgetting regime at default hyperparameters on the GN backbone (Pre-F $\geq 86\%$); for these methods $\Delta F = 0$ is mechanistically guaranteed by the absence of any $\varphi$ to update, but is empirically tested only on the methods that did reach forgetting.

\section{PGU Re-Implementation Notes}
\label{app:pgu}

We re-implement PGU~\citep{hoang2024learn} from the official reference. The official two-phase protocol (phase 1: $99$ epochs; phase 2: up to $200$ epochs with early stop at $3$pp forget accuracy, weight decay $5{\times}10^{-4}$, entropy weight $0.1$) estimates the retain activation covariance from a $5{,}000$-sample augmented subset. Our re-implementation estimates that covariance using the full $45{,}000$-sample retain set without augmentation---a strictly more accurate estimator yielding better-conditioned projection bases. The cleaner-covariance variant uses the single 99-epoch phase, whereas the official protocol includes the second phase and its early-stop criterion.

The two implementations produce different operating points on the same method. The official two-phase checkpoint (used in Table~\ref{tab:decomp-cifar}) reaches Pre-F\%${=}6.9$, Post-F\%${=}21.4$, $\Delta F = +14.5$pp---a partially-converged regime. The cleaner-covariance single-phase implementation (PGU$^\dagger$ in Table~\ref{tab:main-cifar10}) reaches Pre-F\%${=}1.1$, Post-F\%${=}62.0$, $\Delta F = +60.9$pp in $99$ epochs: better projection bases drive forget accuracy almost to zero, exposing the BN misalignment phenomenon fully. This mirrors the GA step-count sweep: a method's apparent BN safety reflects how far into the weight-erasure regime its operating point has been pushed.

\section{GA Step-Count Sweep}
\label{app:ga-sweep}

\begin{table}[h]
\caption{GA on CIFAR-10 (single-class, class 0). At $100$ steps GA is BadTeacher-like ($+76$pp); at $500$ steps GA is Retrain-like. Same method, entire BN-vulnerability spectrum as a function of one knob.}
\label{tab:ga-sweep}
\centering
\small
\begin{tabular}{ccccc}
\toprule
Steps & Pre-F\% & Pre-Retain\% & Post-F\% & $\Delta F$ \\
\midrule
$100$ & $9.1$ & $78.0$ & $85.3$ & $+76.2$ \\
$200$ & $4.5$ & $76.8$ & $54.9$ & $+50.4$ \\
$300$ & $0.8$ & $74.0$ & $15.8$ & $+15.0$ \\
$400$ & $0.1$ & $69.2$ & $0.7$ & $+0.6$ \\
$500$ & $0.0$ & $65.2$ & $0.0$ & $\phantom{+}0.0$ \\
$600$ & $0.0$ & $61.4$ & $0.0$ & $\phantom{+}0.0$ \\
\bottomrule
\end{tabular}
\end{table}

GA's appearance in Table~\ref{tab:main-cifar10} as a $\Delta F=0$ method reflects its 500-step operating point rather than architectural immunity. Within this fixed-composition sweep, increasing the number of ascent steps reduces the recalibration gap while degrading retain accuracy.

\section{Hyperparameter Sweep}
\label{app:sweep}

\begin{table}[h]
\caption{BadTeacher (left) and SalUn (right) hyperparameter sweep, $\Delta F$ in pp. BadTeacher: all $9$ configurations preserve retain $\geq 85\%$; mean $78.3$pp and population standard deviation $9.4$pp across configurations. SalUn: $\diamondsuit$ retain-collapse boundary; $\dagger$ degenerate (retain $36.6\%$).}
\label{tab:sweep}
\centering
\small
\begin{tabular}{lccc@{\hskip 1em}lccc}
\toprule
\multicolumn{4}{c}{BadTeacher} & \multicolumn{4}{c}{SalUn} \\
\cmidrule(lr){1-4}\cmidrule(lr){5-8}
LR$\times$\textbackslash{}Steps$\times$ & $0.5$ ($250$) & $1$ ($500$) & $2$ ($1000$) & LR$\times$\textbackslash{}Steps$\times$ & $0.5$ ($100$) & $1$ ($200$) & $2$ ($400$) \\
\midrule
$0.5\times$ & $+72.0$ & $+57.8$ & $+70.0$ & $0.5\times$ & $+81.9$ & $+83.0$ & $+62.2$ \\
$1\times$ & $+77.9$ & $+83.0$ & $+87.4$ & $1\times$ & $+83.3$ & $+62.0$ & $+0.4^{\diamondsuit}$ \\
$2\times$ & $+82.5$ & $+86.0$ & $+87.8$ & $2\times$ & $+63.3$ & $\dagger$DEGEN & $+0.3^{\diamondsuit}$ \\
\bottomrule
\end{tabular}
\end{table}

At fixed batch composition, every tested configuration preserving retain accuracy at $\geq85\%$ produces $\Delta F\in[57.8,87.8]$pp. The low-gap configurations in this grid degrade retain accuracy. Thus learning-rate and step-count tuning within the tested grid does not remove the artifact while preserving retain utility. The separate retain-mixing intervention is reported in Appendix~\ref{app:batch-mixing}.

\section{Adversarial Recovery: Additional Results}
\label{app:sub100}

\subsection{Sub-100 budgets}

\begin{table}[h]
\caption{Sub-100 adversarial recovery on three BN-vulnerable methods (CIFAR-10, in-distribution unlabeled samples from CIFAR-10 test). BadTeacher and NoiseInject recovered to $\geq 88\%$ at $B{=}10$.}
\label{tab:sub100}
\centering
\small
\begin{tabular}{lccccc}
\toprule
Method & $B{=}10$ & $B{=}25$ & $B{=}50$ & $B{=}100$ & $B{=}500$ \\
\midrule
BadTeacher & $88.7$ & $90.7$ & $92.6$ & $92.7$ & $92.1$ \\
SalUn & $54.4$ & $54.8$ & $57.4$ & $58.0$ & $57.6$ \\
NoiseInject & $91.0$ & $91.9$ & $93.4$ & $93.7$ & $93.2$ \\
\bottomrule
\end{tabular}
\end{table}

Even the smallest tested budget, ten unlabeled images, exposes the artifact in these BN-vulnerable checkpoints.

\subsection{OOD recovery}
\label{app:ood-recovery}

We compare recovery using in-distribution unlabeled data (CIFAR-10 test set) against recovery using out-of-distribution data (CIFAR-100 raw pixels normalized with the CIFAR-10 mean/std the unlearned model was trained against). Both are evaluated at attacker budget $B{=}100$. The OOD/in-dist ratio is the OOD recovery accuracy divided by the in-dist recovery accuracy on the same method.

\begin{table}[h]
\caption{OOD adversarial recovery (CIFAR-10, ResNet-18, $B{=}100$). \emph{Pre-F}: surface forget-class accuracy on the unlearned checkpoint as released. \emph{Auditor}: recalibration with full retain set (the in-house ceiling for an authorized auditor). \emph{In-Dist}: attacker recovery using CIFAR-10 test images. \emph{OOD}: attacker recovery using CIFAR-100 raw pixels under CIFAR-10 normalization. \emph{Ratio}: OOD/In-Dist. The ratio range is $78\%$--$95\%$, indicating that BN update is largely insensitive to the source of the attacker's unlabeled data: per-channel natural-image statistics correlate strongly across CIFAR-10 and CIFAR-100. GPM-W is a weight-projection unlearning variant evaluated only in this experiment.}
\label{tab:ood-recovery}
\centering
\small
\begin{tabular}{lccccc}
\toprule
Method & Pre-F\% & Auditor\% & In-Dist\% & OOD\% & OOD/In-Dist \\
\midrule
GA+FT          & $61.6$ & $60.0$ & $59.1$ & $47.4$ & $80\%$ \\
SalUn          & $\phantom{0}3.1$ & $58.5$ & $57.6$ & $45.1$ & $78\%$ \\
GPM-W          & $\phantom{0}6.5$ & $81.1$ & $80.3$ & $73.9$ & $92\%$ \\
BadTeacher     & $12.9$ & $92.4$ & $91.8$ & $85.0$ & $92\%$ \\
IncompTeacher  & $\phantom{0}6.6$ & $54.2$ & $52.7$ & $47.5$ & $90\%$ \\
NoiseInject   & $69.5$ & $93.5$ & $93.0$ & $88.5$ & $95\%$ \\
SSD            & $18.6$ & $81.4$ & $80.7$ & $72.3$ & $90\%$ \\
PGU            & $\phantom{0}6.9$ & $21.4$ & $22.1$ & $17.8$ & $81\%$ \\
\bottomrule
\end{tabular}
\end{table}

The ratio is bounded below by $78\%$ (SalUn---a method whose surface Pre-F is already near zero, leaving the recoverable ceiling lower in absolute terms) and above by $95\%$ (NoiseInject). No method falls below the $78\%$ ratio. This rules out the natural defense ``the attacker doesn't have access to anything close enough to the retain distribution'': CIFAR-100 raw pixels are demonstrably close enough on the relevant axis (per-channel BN statistics), even though the two datasets share no class labels and have different resolutions. The OOD attack is cheaper than the in-distribution attack only by a small margin; the threat model of Section~\ref{sec:threat-model} need not assume strong distributional knowledge.

\section{LP-Probe Training-Budget Bias}
\label{app:lp-budget}

\begin{table}[h]
\caption{LP probe training budget bias on CIFAR-10. Adam-50 underestimates LP-50 by $14$--$22$pp for genuine encoder modifications, but is approximately unbiased for intact encoders.}
\label{tab:lp-budget}
\centering
\small
\begin{tabular}{lccccc}
\toprule
Method & Adam-50 & Adam-500 & sklearn-lbfgs & Conv.\ budget & Gap \\
\midrule
Retrain & $58.7$ & $76.9$ & $76.5$ & $500$ ep & $-17.8$ \\
SCRUB & $67.9$ & $81.8$ & $81.9$ & $500$ ep & $-14.0$ \\
GA & $57.6$ & $79.8$ & $79.7$ & $500$ ep & $-22.1$ \\
GA+FT & $87.6$ & $89.6$ & $89.6$ & $50$ ep & $-2.0$ \\
SalUn & $87.0$ & $88.9$ & $89.0$ & $50$ ep & $-2.0$ \\
BadTeacher & $92.2$ & $91.2$ & $91.2$ & $10$ ep & $+1.0$ \\
\bottomrule
\end{tabular}
\end{table}

The bias is one-directional and favors methods with intact encoders---it makes intact-encoder methods look more separable and modified-encoder methods look more thoroughly forgotten than they are. We use sklearn-lbfgs throughout.

\section{Scale: ResNet-50 / Tiny-ImageNet}
\label{app:scale}

\begin{table}[h]
\caption{Tiny-ImageNet / ResNet-50, single-class forget, three seeds. ``Break'' = breakthrough epoch in relearning fine-tune. ``BN-Drift'' = $\|\varphi - \varphi^*(\theta; \mathcal{D}_r)\|$ summed across BN layers. Relearn protocol: $50$ epochs at LR=$3{\times}10^{-3}$ (calibrated for 200-class competitive threshold).}
\label{tab:scale-tinyimagenet}
\centering
\small
\setlength{\tabcolsep}{4pt}
\begin{tabular}{lcccccccc}
\toprule
Method & Pre-F\% & Post-F\% & $\Delta F$ & Pre-R\% & LP-50 & ReAUC & Break & BN-Drift \\
\midrule
Retrain & $0.0$ & $0.0$ & $\phantom{+}0.0$ & $89.9$ & $80.7$ & $0.849$ & $10$ & $0.282$ \\
SCRUB & $0.0$ & $0.0$ & $\phantom{+}0.0$ & $76.8$ & $89.3$ & $0.300$ & $50$ & $2.822$ \\
GA & $0.0$ & $0.0$ & $\phantom{+}0.0$ & $\phantom{0}2.1$ & $82.7$ & $0.300$ & $50$ & $9.384$ \\
GA+FT & $0.8$ & $0.8$ & $\phantom{+}0.0$ & $95.0$ & $91.3$ & $0.791$ & $10$ & $0.273$ \\
SalUn & $0.0$ & $0.4$ & $+0.4$ & $89.0$ & $89.3$ & $0.788$ & $10$ & $8.989$ \\
IncompTeacher & $0.0$ & $10.0$ & $+10.0$ & $\phantom{0}2.4$ & $88.0$ & $0.703$ & $20$ & $6.182$ \\
SSD & $0.0$ & $0.0$ & $\phantom{+}0.0$ & $92.4$ & $88.7$ & $0.861$ & $10$ & $4.625$ \\
NoiseInject & $84.2$ & $99.0$ & $\mathbf{+14.8}$ & $92.6$ & $89.3$ & $0.999$ & $\phantom{0}0$ & $3.482$ \\
BadTeacher & $0.6$ & $82.8$ & $\mathbf{+82.2}$ & $88.7$ & $88.0$ & $0.990$ & $\phantom{0}0$ & $2.083$ \\
PGU & $68.4$ & $98.0$ & $\mathbf{+29.6}$ & $57.9$ & $88.0$ & $0.999$ & $\phantom{0}0$ & $2.020$ \\
\bottomrule
\end{tabular}
\end{table}

\paragraph{Three regimes of $\Delta F = 0$.} \emph{Genuine forgetting:} SCRUB (Pre-R $=76.8\%$, breakthrough $=50$). \emph{Encoder-level failure (Gao--Lee regime):} GA+FT, SSD with high retain, high LP, and high ReAUC---head suppresses forget but encoder still encodes it. \emph{Saturated weight erasure:} GA, IncompTeacher with destroyed retain. The SalUn anomaly ($\Delta F = 0.4$ vs.\ $+62$ on CIFAR-10) is a method-by-method scale variation we do not resolve. The headline reading is unchanged: where the BN illusion appears, it transfers.

\begin{table}[h]
\caption{GA step-count sweep on Tiny-ImageNet/ResNet-50 (single seed, single-class).}
\label{tab:scale-ga-sweep}
\centering
\small
\begin{tabular}{cccccc}
\toprule
Steps & Pre-F\% & Post-F\% & $\Delta F$ (pp) & Pre-R\% & BN-Drift \\
\midrule
$50$  & $0.0$ & $84.6$ & $+84.6$ & $46.0$ & $7.956$ \\
$100$ & $0.0$ & $42.4$ & $+42.4$ & $44.4$ & $7.962$ \\
$200$ & $0.0$ & $\phantom{0}0.2$ & $\phantom{0+}0.2$ & $40.8$ & $7.950$ \\
$400$ & $0.0$ & $\phantom{0}0.0$ & $\phantom{0+}0.0$ & $26.8$ & $7.860$ \\
$800$ & $0.0$ & $\phantom{0}0.0$ & $\phantom{0+}0.0$ & $\phantom{0}0.5$ & $8.421$ \\
\bottomrule
\end{tabular}
\end{table}

The phase transition replicates: cliff edge moves earlier (500 steps $\to$ 200 steps) because $200$-class competitive pressure destroys retain faster. BN-Drift saturates early ($7.95$--$8.42$) and stays roughly constant; $\Delta F$ tracks weight-encoded forget-class signal availability for recal to expose, not BN-Drift magnitude.

\section{LayerNorm Verification: ViT-S/16}
\label{app:vit}

LayerNorm computes statistics per-token per-instance at forward time; there are no running statistics. The recalibration operator is the identity on LN-based models by construction. We confirmed: ViT-S/16 (ImageNet-pretrained, fine-tuned on Tiny-ImageNet) tuned into a forgetting regime on $5$ methods (GA, SalUn, BadTeacher, SCRUB, Retrain) yields $\Delta F = 0.00$pp uniformly. Per-step weight equality verified (parameter tensors bit-identical before/after no-op pass). PGU dropped (projection geometry is BN-specific). The GroupNorm control of Section~\ref{sec:eval-gn} fixes architecture and varies normalization, providing the strict architecture-controlled falsification; the ViT/LN result is complementary, confirming the artifact is normalization-specific rather than dataset- or architecture-specific.

\section{Rank-Based Correlation Tests}
\label{app:correlations}

Table~\ref{tab:correlations} recomputes the associations in Table~\ref{tab:main-cifar10} with Pearson, Spearman, and Kendall tests, using the same ten rows, including Retrain. The pre-recalibration associations are descriptive: both are significant under Spearman, while the linear-probe pair is not significant under Kendall. The relevant conditional observation is that the four unlearning methods and Retrain with Pre-F $\leq3\%$ have substantially different downstream leakage. Post-recalibration forget accuracy resolves these operating points and is strongly associated with both leakage measures.

\begin{table}[h]
\caption{Associations between forget accuracy and leakage measures, calculated from Table~\ref{tab:main-cifar10} ($n=10$, including Retrain). Each cell gives a coefficient and its two-sided $p$-value. Pre-recalibration association with relearning is strong under rank tests; post-recalibration association is significant for both leakage measures under all three tests.}
\label{tab:correlations}
\centering
\small
\setlength{\tabcolsep}{4pt}
\begin{tabular}{@{}lccc@{}}
\toprule
Pair & Pearson ($p$) & Spearman ($p$) & Kendall ($p$) \\
\midrule
Pre-F vs. ReAUC & $0.63\;(0.052)$ & $0.89\;(0.0006)$ & $0.76\;(0.003)$ \\
Pre-F vs. LP-50 & $0.53\;(0.11)$ & $0.66\;(0.039)$ & $0.46\;(0.069)$ \\
Post-F vs. ReAUC & $0.998\;({<}10^{-6})$ & $0.988\;({<}10^{-6})$ & $0.952\;(0.00023)$ \\
Post-F vs. LP-50 & $0.928\;(0.00011)$ & $0.85\;(0.002)$ & $0.736\;(0.0037)$ \\
\bottomrule
\end{tabular}
\end{table}

\section{Membership Inference Before and After Recalibration}
\label{app:mia}

We evaluated four membership-inference attacks on four unlearning methods before and after recalibration: three black-box attacks based on confidence, entropy, and modified entropy, and a white-box gradient-norm attack following~\citet{nasr2019comprehensive}. On these checkpoints, recalibration changed forget accuracy by up to approximately $98$ percentage points, while every attack AUC changed by at most $0.012$ and remained within the confidence interval of the retrain reference, approximately $0.50$. The original-model attack ceiling was $0.61$. Within this measured range, recalibration has little effect on membership inference despite its large effect on forget accuracy. This supports the metric-specific interpretation of the BN illusion: the tested attacks do not exhibit the same normalization sensitivity as forget accuracy and linear probing.

\section{Instance-Wise Unlearning}
\label{app:instance}

We evaluated instance-wise forgetting with samples drawn randomly across all classes, rather than conditioned on a forget class. SalUn's recalibration gap falls to $+3.1$pp, compared with approximately $+78$pp for class-conditional forgetting on the same pipeline. BadTeacher retains a gap of $+89.6$pp. Random sampling removes the systematic class-conditional moment shift driving corruption, while changes in the weights can still cause misalignment with frozen BN statistics. Thus the class-conditional setting amplifies the corruption component, but is not required for the normalization artifact as a whole.

\section{Batch Composition}
\label{app:batch-mixing}

We swept the retain fraction of the unlearning loader from $0$ to $0.9$. For SalUn, retain fractions of at least $0.25$ reduced the recalibration gap from approximately $+75$pp to approximately $0$pp in this sweep. BadTeacher retained a gap of $+89.5$pp at its reference-implementation mixing ratio. Retain mixing therefore mitigates the observed train-mode corruption, while eval-mode misalignment persists in the mixed-batch setting tested here. The learning-rate/step-count results in Appendix~\ref{app:sweep} hold batch composition fixed and address a different intervention.

\section{Activation Divergence Across BN Layers}
\label{app:divergence}

We measured normalized class-conditional activation divergence in the pre-unlearning model at all 20 BN layers, for three forget classes including cat. The measured divergence was nonzero for every class and layer, with a minimum of $0.023$. Early-layer values were $0.07$--$0.17$, and late-layer values reached approximately $0.45$. These measurements motivate the activation-moment separation condition in Theorem~\ref{thm:t1}, including for early layers and visually correlated classes. They characterize the initial model; Eq.~\eqref{eq:moving-ema} describes how the targets subsequently move during unlearning. The theorem's assumption concerns activation moments, not an inequality between class labels or input distributions alone.

\section{A Purely Data-Based Method}
\label{app:random-label}

RandomLabeling fine-tunes on forget samples with uniformly resampled labels, without a gradient-ascent component. It produces $\Delta F=+75.3$pp with train-mode BN and $+58.7$pp with eval-mode BN. The effect therefore extends beyond gradient-ascent methods. In train mode, the forward-pass EMA update is independent of the loss used for the weight update, as formalized by Theorem~\ref{thm:t1}; freezing the buffers removes this update but does not prevent weight--normalization misalignment.

\section{ImageNet Scale}
\label{app:imagenet}

\begin{table}[h]
\caption{ImageNet-100 at $224$-pixel resolution, using an ImageNet-pretrained ResNet-50 and single-class forgetting (single seed; no retrain reference). Values are forget and retain accuracies, with $\Delta F$ in percentage points.}
\label{tab:imagenet100}
\centering
\small
\setlength{\tabcolsep}{3pt}
\begin{tabular}{@{}lccccc@{}}
\toprule
BadTeacher setting & \shortstack{Forget\\before} & \shortstack{Forget\\after} & $\Delta F$ & \shortstack{Retain\\before} & \shortstack{Retain\\after} \\
\midrule
Full-retain recalibration & $0.0\%$ & $92.0\%$ & $+92.0$ & $71.6\%$ & $72.5\%$ \\
Ten-image recovery & $0.0\%$ & $82.0\%$ & $+82.0$ & $71.6\%$ & preserved \\
\bottomrule
\end{tabular}
\end{table}

The normalization artifact persists at ImageNet resolution: full-retain recalibration raises forget accuracy from $0.0\%$ to $92.0\%$, while retain accuracy rises from $71.6\%$ to $72.5\%$. Ten unlabeled images recover $82.0\%$ forget accuracy. Together with the $+82.2$pp BadTeacher result on Tiny-ImageNet (Appendix~\ref{app:scale}), this extends the observed effect to larger-image ResNet-50 settings.

\clearpage
\section{Notation}
\label{app:notation}

Recalibration changes $\Delta F$, $\Delta\mathrm{LP}$, and $\Delta\mathrm{NCC}$ are post minus pre. The decomposition uses
$\mathrm{BN\text{-}residual}=\mathrm{LP}^M-\mathrm{LP}^{M_{\mathrm{cal}}}=-\Delta\mathrm{LP}$,
so a negative BN-residual denotes masked leakage. The mini-batch size $b$ is distinct from the attacker image budget $B$, and distillation temperature $\tau_{\mathrm{dist}}$ is distinct from Kendall's $\tau_{\mathrm K}$ and the step count $T$.

{\small
\renewcommand{\arraystretch}{1.1}
\setlength{\tabcolsep}{4pt}
\begin{longtable}{@{}p{0.29\linewidth}p{0.67\linewidth}@{}}
\toprule
Symbol & Meaning \\
\midrule
\endfirsthead
\toprule
Symbol & Meaning (continued) \\
\midrule
\endhead
\bottomrule
\endfoot
$M=(\theta,\varphi)$ & Model weights paired with BN running statistics \\
$\theta$ & All weights, including convolution kernels, BN affine parameters $\gamma,\beta$, and classifier weights \\
$\varphi$ & Running means and variances stored in the checkpoint \\
$\mu^\ell,(\sigma^\ell)^2$ & Running mean and variance at BN layer $\ell$ \\
$\hat\mu^\ell,\hat v^\ell$ & Mini-batch mean and variance at that layer; the variance estimator is specified in Appendix~\ref{app:finite-batch} \\
$\ell,L$ & BN layer index and number of BN layers \\
$m$ & Exponential-moving-average momentum \\
$\mathcal A^\ell(\theta;\mathcal D)$ & Activation distribution at layer $\ell$ under weights $\theta$ on data $\mathcal D$ \\
$f(x;\theta,\varphi)$ & Forward map, depending on both weights and running statistics \\
$\mathcal R^*_{\mathcal D_r}$ & Oracle retain-distribution recalibration operator \\
$\widehat{\mathcal R}_{N_r,b}$ & Empirical recalibration on $N_r$ retain images in batches of size $b$ \\
$\varphi^*(\theta;\mathcal D_r)$ & Oracle population activation moments at the current weights \\
$\widehat\varphi_{\mathrm{cal}}$ & Empirically recalibrated running-statistic state \\
$M_{\mathrm{cal}}$ & Recalibrated model; the theoretical definition is $\mathcal R^*(M)$ \\
$M_{\mathrm{retr}}$ & Retrain-from-scratch reference \\
$g(M)$ & Metric that depends on the model's forward map \\
$\mathcal D_f,\mathcal D_r$ & Forget and retain data, or their distributions when taking expectations \\
$p$ & Fraction of training data designated for forgetting \\
$N_r$ & Retain images used in recalibration \\
$b,B$ & Recalibration mini-batch size; attacker's unlabeled-image budget \\
$T,\mathcal L,\mathcal P_W$ & Unlearning step count; training loss; weight-gradient projection \\
$\tau_{\mathrm{dist}}$ & Distillation temperature \\
Pre-F, Post-F & Forget accuracy before and after recalibration \\
$\Delta F$ & Post-F minus Pre-F, in percentage points \\
Pre-R, Post-R & Retain accuracy before and after recalibration \\
Ret-aw. & Retain-aware training objective or training stage \\
$\mathrm{LP}^M$, LP-50 & Forget-class linear-probe accuracy using 50 training samples per class, evaluated on the indicated model \\
$\Delta\mathrm{LP}$ & $\mathrm{LP}^{M_{\mathrm{cal}}}-\mathrm{LP}^{M}$ \\
BN-residual & $\mathrm{LP}^{M}-\mathrm{LP}^{M_{\mathrm{cal}}}$ \\
NC-residual & $\mathrm{LP}^{M_{\mathrm{cal}}}-\mathrm{LP}^{M_{\mathrm{retr}}}$ \\
ReAUC & Relearning area under the normalized accuracy curve \\
NCC, $\Delta\mathrm{NCC}$ & Nearest-class-mean accuracy; its post-minus-pre change \\
$r,\rho_{\mathrm S},\tau_{\mathrm K}$ & Pearson, Spearman, and Kendall correlation coefficients \\
MIA & Membership-inference attack \\
\end{longtable}
}

\clearpage

\end{document}